\documentclass{article}

\usepackage[final]{neurips_2020}

\usepackage{microtype}
\usepackage{graphicx}
\usepackage{subcaption,geometry}
\newsavebox{\largestimage}
\usepackage{wrapfig}

\usepackage{booktabs}
\PassOptionsToPackage{hyphens}{url}
\usepackage{hyperref}

\usepackage{mathtools}
\usepackage{amsmath,amsthm,amssymb,dsfont}

\usepackage[svgnames]{xcolor}
\usepackage[normalem]{ulem}
\usepackage{soul}
\usepackage{setspace}
\usepackage[ruled,vlined]{algorithm2e}

\usepackage[inline]{enumitem}
\usepackage[mathscr]{euscript}

\usepackage[final]{fixme}
\fxsetup{inline, nomargin, theme=color, mode=multiuser}

\newtheorem{thmlem}{Lemma}

\newtheorem{thmthm}{Theorem}

\newtheorem{thmasmp}{Assumption}
\newtheorem*{thmidasmp*}{Identifying assumptions}

\theoremstyle{definition}
\newtheorem{thmex}{Example}
\newtheorem{thmrem}{Remark}

\newenvironment{proofsketch}{\emph{Proof sketch.}}{\hfill$\square$}
    
\DeclareMathOperator*{\argmax}{arg\,max}
\DeclareMathOperator*{\argmin}{arg\,min}

\def\E{\mathbb{E}}

\def\cH{\mathcal{H}}
\def\cI{\mathcal{I}}
\def\cY{\mathcal{Y}}
\def\cX{\mathcal{X}}
\def\cF{\mathcal{F}}
\def\cG{\mathcal{G}}
\def\cA{\mathcal{A}}

\def\cS{\mathcal{S}}
\def\cU{\mathcal{U}}

\def\cZ{\mathcal{Z}}
\def\olcA{\overline{\mathcal{A}}}

\def\olcY{\overline{\mathcal{Y}}}

\def\olA{\overline{A}}

\def\olY{\overline{Y}}
\def\ola{\overline{a}}

\def\oly{\overline{y}}

\def\ds1{\mathds{1}}

\newcommand\indep{\protect\mathpalette{\protect\independenT}{\perp}}
\def\independenT#1#2{\mathrel{\rlap{$#1#2$}\mkern2mu{#1#2}}}

\def\stop{\textnormal{\small \sc \textbf{stop}}}

\def\cAstop{\cA\cup \{\stop{}\}}

\def\cg{{\sc CG}}
\def\cdp{{\sc CDP}}
\def\ndp{{\sc NDP}}

\title{Learning to search efficiently for causally near-optimal treatments}

\author{%
  Samuel H\aa{}kansson\thanks{This work was completed while the author was affiliated with Chalmers University of Technology.}  \\
  University of Gothenburg\\
  \texttt{samuel.hakansson@gu.se} \\
  \And
  Viktor Lindblom  \\
  Chalmers University of Technology\\
  \texttt{viklindb@student.chalmers.se} \\
  \And
  Omer Gottesman\thanks{This work was completed while the author was affiliated with Harvard University.} \\
  Brown University\\
  \texttt{omer\_gottesman@brown.edu} \\
  \And
  Fredrik D. Johansson  \\
  Chalmers University of Technology\\
  \texttt{fredrik.johansson@chalmers.se} \\
}

\begin{document}

\maketitle

%
%
\begin{abstract}
Finding an effective medical treatment often requires a search by trial and error. Making this search more efficient by minimizing the number of unnecessary trials could lower both costs and patient suffering. We formalize this problem as learning a policy for finding a near-optimal treatment in a minimum number of trials using a causal inference framework. We give a model-based dynamic programming algorithm which learns from observational data while being robust to unmeasured confounding. To reduce time complexity, we suggest a greedy algorithm which bounds the near-optimality constraint. The methods are evaluated on synthetic and real-world healthcare data and compared to model-free reinforcement learning. We find that our methods compare favorably to the model-free baseline while offering a more transparent trade-off between search time and treatment efficacy. 
\end{abstract}

%
%
\section{Introduction}
Finding a good treatment for a patient often involves trying out different options before a satisfactory one is found~\citep{murphy2007customizing}. If the first-line drug is ineffective or has severe side-effects, guidelines may suggest it is replaced by or combined with another drug~\citep{singh20162015}. These steps are repeated until an effective combination of drugs is found or all options are exhausted, a process which may span several years~\citep{national2010depression}. A long search adds to patient suffering and postpones potential relief. It is therefore critical that this process is made as time-efficient as possible.

We formalize the search for effective treatments as a policy optimization problem in an unknown decision process with finite horizon~\citep{garcia1998learning}. This has applications also outside of medicine: For example, in recommendation systems, we may sequentially propose new products or services to users with the hope of finding one that the user is interested in. Our goal is to perform as few  trials as possible until the probability that there are untried actions which are significantly better is small---i.e., \emph{a near-optimal action has been found with high probability}. Historical observations allow us to transfer knowledge and perform this search more efficiently for new subjects. As more actions are tried and their outcomes observed, our certainty about the lack of better alternatives increases. Importantly, even a failed trial may provide information that can guide the search policy.

In this work, we restrict our attention to actions whose outcomes are stationary in time. This implies both that repeated trials of the same action have the same outcome and that past actions do not causally impact the outcome of future actions. The stationarity assumption is justified, for example, for medical conditions where treatments manage symptoms but do not alter the disease state itself, or where the impact of sequential treatments is known to be additive. In such settings, past actions and outcomes may help predict the outcomes of future actions without having a causal effect on them.

%
%
We formalize learning to search efficiently for causally effective treatments as off-policy optimization of a policy which finds a near-optimal action for new contexts after as few trials as possible. Our setting differs from those typical of  reinforcement or bandit learning~\citep{sutton1998introduction}:
\begin{enumerate*}[label=(\roman*)]
\item Solving the problem relies on transfer of knowledge from observational data.
\item The stopping (near-optimality) criterion depends on a model of unobserved quantities.
\item The number of trials in a single sequence is bounded by the number of available actions.
\end{enumerate*}
We address identification of an optimal policy using a causal framework, accounting for potential confounding. We give a dynamic programming algorithm which learns policies that satisfy a transparent constraint on near-optimality  for a given level of confidence, and a greedy approximation which satisfies a bound on this constraint. We show that greedy policies are sub-optimal in general, but that there are settings where they return policies with informative guarantees. In experiments, including an application derived from antibiotic resistance tests, our algorithms successfully learn efficient search policies and perform favorably to baselines.

%
%
\section{Related work}
Our problem is related to the bandit literature, which studies the search for optimal actions through trial and error~\citep{lattimore2020bandit}, and in particular to contextual bandits~\citep{abe2003reinforcement,chu2011contextual}. In our setting, a very small number of actions is evaluated, with the goal of terminating search as early as possible. This is closely related to the fixed-confidence variant of best-arm identification~\citep[Chapter 33.2]{lattimore2020bandit}, in which only exploration is performed. To solve this problem without trying each action at least once, we rely on transferring knowledge from previous trials. This falls within the scope of transfer and meta learning. \citet{liao2020personalized} explicitly tackled pooling knowledge across patients data to determine an optimal treatment policy in an RL setting and  \citet{maes2012meta} devised methods for meta-learning of exploration policies for contextual bandits. A notable difference is that we assume that outcomes of actions are stationary in time. We leverage this both in model identification and policy optimization.

Experiments continue to be the gold standard for evaluating adaptive treatment strategies~\citep{nahum2012experimental}. However, these are not always feasible due to ethical or practical constraints. We approach our problem as causal estimation from observational data~\citep{rosenbaum2010design,robins2000marginal}, or equivalently, as off-policy policy optimization and evaluation~\citep{precup2000eligibility,kallus2018optimal}.  Unlike many works, we do not fully rely on ignorability---that all confounders are measured and may be adjusted for. \citet{zhang2019near} recently studied the non-ignorable setting but allowed for limited online exploration. In this work we aim to bound the effect of unmeasured confounding rather than to eliminate it using experimental evidence.

Our problem is closely related to active learning~\citep{lewis1994sequential}, which has been used to develop testing policies that minimize the expected number of tests performed before an underlying hypothesis is identified. For a known distribution of hypotheses, finding an optimal  policy is NP-hard~\citep{chakaravarthy2007decision}, but there exists greedy algorithms with approximation guarantees~\citep{golovin2010near}. In our case, (i) the distribution is unknown, and (ii) hypotheses (outcomes) are only partially observed. 
Our problem is also related to optimal stopping~\citep{jacka1991optimal} of processes but differs in that the process is controlled by our decision-making agent.

%
%
\section{Learning to search efficiently for causally near-optimal treatments}%
\label{sec:background}%
We consider learning policies $\pi \in \Pi$ that search over a set of actions $\cA \coloneqq \{1, ..., k\}$ to find an action $a\in \cA$ such that its outcome $Y(a) \in \cY$ is near-optimal. When such an action is found, the search should be terminated as early as possible using a special stop action, denoted $a=\stop{}$. Throughout, a high outcome is assumed to be preferred and we often refer to actions as ``treatments''. The potential outcome $Y(a)$ may vary between subjects (contexts) depending on \emph{baseline covariates} $X \in \cX \subseteq \mathbb{R}^d$ and unobserved factors. When a search starts, all potential outcomes $\{Y(a) : a\in \cA\}$ are unobserved, but are successively revealed as more actions are tried, see the illustration in Figure~\ref{fig:timeline}. To guide selection of the next action, we learn from observational data of previous subjects. 

Historical searches are observed through covariates $X$ and a sequence of $T$ \emph{action}-\emph{outcome} pairs $(A_1, Y_1), ..., (A_{T}, Y_{T})$. Note the distinction between interventional and observational outcomes; $Y_s(a)$ represents the \emph{potential} outcome of performing action $a$ at time $s$~\citep{rubin2005causal}. We assume that $y \in \cY$ are discrete, although our results may be generalized to the continuous case. Sequences of $s$ random variables $(A_1, ..., A_s)$ are denoted with a bar and subscript, $\olA_s \in \olcA_s$, and $H_s = (X, \olA_s, \olY_s) \in \cH_s \coloneqq \cX \times \{\olcA_s \times \olcY_s\}$ denotes the history up-to time $s$, with $H_0 = (X, \emptyset, \emptyset)$.  With slight abuse of notation, $a \in h$ means that $a$ was used in $h$ and $(a,y)\in h$ that it had the outcome $y$. $|h|$ denotes the number of trials in $h$. The set of histories of at most $k$ actions is denoted $\cH = \cup_{s=1}^k \cH_s$. 
Termination of a sequence is indicated by the sequence length $T$ and may either be the result of finding a satisfactory treatment or due to censoring. Hence, the full set of potential outcomes is not observed for most subjects. Observations are distributed according to $p(X, T, \olA_T, \olY_T)$.

We optimize a deterministic policy $\pi$ which suggests an action $a$ following observed history $h$, starting with $(x, \emptyset, \emptyset)$, or terminates the sequence. Formally, $\pi \in \Pi \subseteq \{\cH \rightarrow \cAstop \}$. Taking the action $\pi(h_s)= \stop{}$ at a time point $s$ implies that $T=s$. 
Let $p_\pi(X, \olA, \olY, T)$ be the distribution in which actions are drawn according to the policy $\pi$.
For a given \emph{slack parameter} $\epsilon \geq 0$ and a \emph{confidence parameter} $\delta \geq 0$, we wish to solve the following problem.
\begin{equation}
\begin{aligned}
& \underset{\pi \in \Pi}{\text{minimize}} 
& & \E_{X, \olY, \olA, T \sim p_\pi}[T] \\
& \text{subject to} 
& & \Pr\left[\max_{a \not \in \olA_t} Y(a) > \max_{(a,y) \in h} y + \epsilon \Bigm\vert H_t=h, T=t \right] \leq \delta, \;\;\; \forall t\in \mathbb{N}, h\in \cH_t
\end{aligned}\label{eq:mainprob}
\end{equation}

In \eqref{eq:mainprob}, the objective equals the expected search length under $\pi$ and the constraint enforces that termination occurs only when there is low probability that a better action will be found among the unused alternatives. Note that if $\max_{y \in \cY} y$ is known and is in  $\olY_s$, the constraint is automatically satisfied at $s$.
To evaluate the constraint, we need a model of unobserved potential outcomes. This is dealt with in Section~\ref{sec:identification}. We address optimization of \eqref{eq:mainprob} for a known model in Section~\ref{sec:optimization}.

%
%
\section{Causal identification and estimation of optimality conditions}
\label{sec:identification}

\begin{figure}
    \centering
    \includegraphics[width =0.6\textwidth]{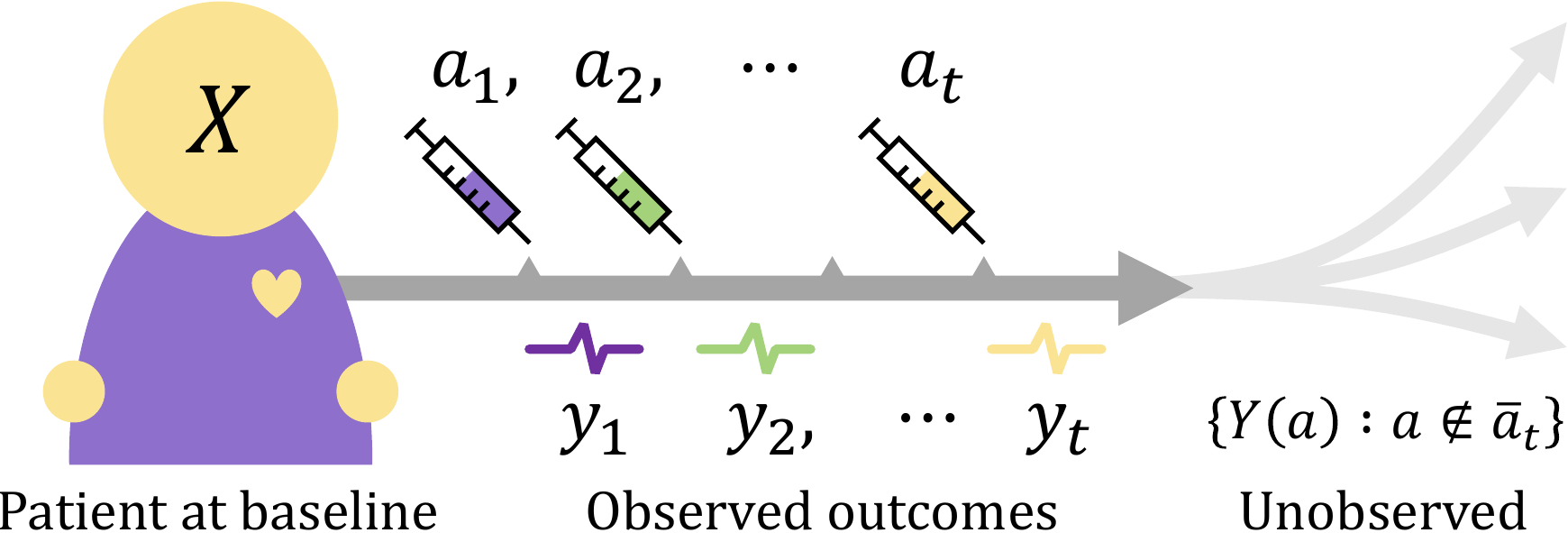}
    \caption{Illustration of the observed sequence of treatments $\ola_t = (a_1, ..., a_t)$ and outcomes  $\oly_t = (y_1, ..., y_t)$ for a patient, and the problem of estimating the outcome of possible future treatments.}
    \label{fig:timeline}
\end{figure}

Our assumed causal model for observed data is illustrated graphically in Figure~\ref{fig:causalmodel}. Most notably, the graph defines the causal structure between actions and outcomes---previous actions $A_1, ..., A_{s-1}$ and outcomes $Y_1, ..., Y_{s-1}$ are assumed to have no direct causal effect on future outcomes $Y_s, ..., Y_T$. To allow for correlations between outcomes, we posit the existence of counfounders $X$ (observed) and $U$ (unobserved) and an unobserved moderator $Z$. All other variables are assumed exogenous.

To evaluate the near-optimality constraint and solve \eqref{eq:mainprob}, we must identify the probability 
\begin{equation}%
\label{eq:nearopt}%
\rho(h) \coloneqq \Pr[\max_{a \not\in h} Y(a) > \max_{(a,y) \in h} y + \epsilon \mid H=h]~,
\end{equation}
with the convention that $\max_{(a,y) \in h_0} y = -\infty$ if $|h_0| = 0$. Henceforth, let $\cA_{-h} = \{a \in \cA  : a \not\in h_s\}$ denote the set of untried actions at $h$ and let $\cS(\cA_{-h})$ be all permutations of the elements in $\cA_{-h}$.

We state assumptions sufficient for identification of $\rho(h)$ below. Throughout this work we assume that consistency, stationarity of outcomes and positivity always hold, and provide identifiability results both when ignorability holds (Section \ref{sec:ignorable}) and when it is violated (Section \ref{sec:nonignorable}). 
\begin{thmidasmp*}
 Define $\olA_{s+1:k} = (A_{s+1}, ..., A_k)$. Under the observational distribution $p$, and evaluation distribution $p_\pi$, for all  $\pi \in \Pi$, $h\in \cH_s$, and $s,r \in \mathbb{N}$, we assume
\textnormal{
\begin{enumerate}
    \item \textbf{Consistency:} $Y_{s} = Y_s(A_s)$
    \label{asmp:consistency}
    \item \textbf{Stationarity:} $Y_{s}(a) = Y_{r}(a) =: Y(a)$
    \label{asmp:stationarity}
    \item \textbf{Positivity:} \;\;\; $\exists \ola \in \cS(\cA_{-h_s}) \colon p_\pi(H_s=h_s) > 0 \implies p(\olA_{s+1:k} = \ola \mid H_s=h_s) > 0$ 
    \label{asmp:positivity}
    \item \textbf{Ignorability:}\; $Y_s(a) \indep A_s \mid H_{s-1}$
    \label{asmp:ignorability}
\end{enumerate}
}
\end{thmidasmp*}

\begin{figure}
\centering
\begin{subfigure}[t]{0.48\textwidth}
    \centering
    \includegraphics[width=0.9\textwidth]{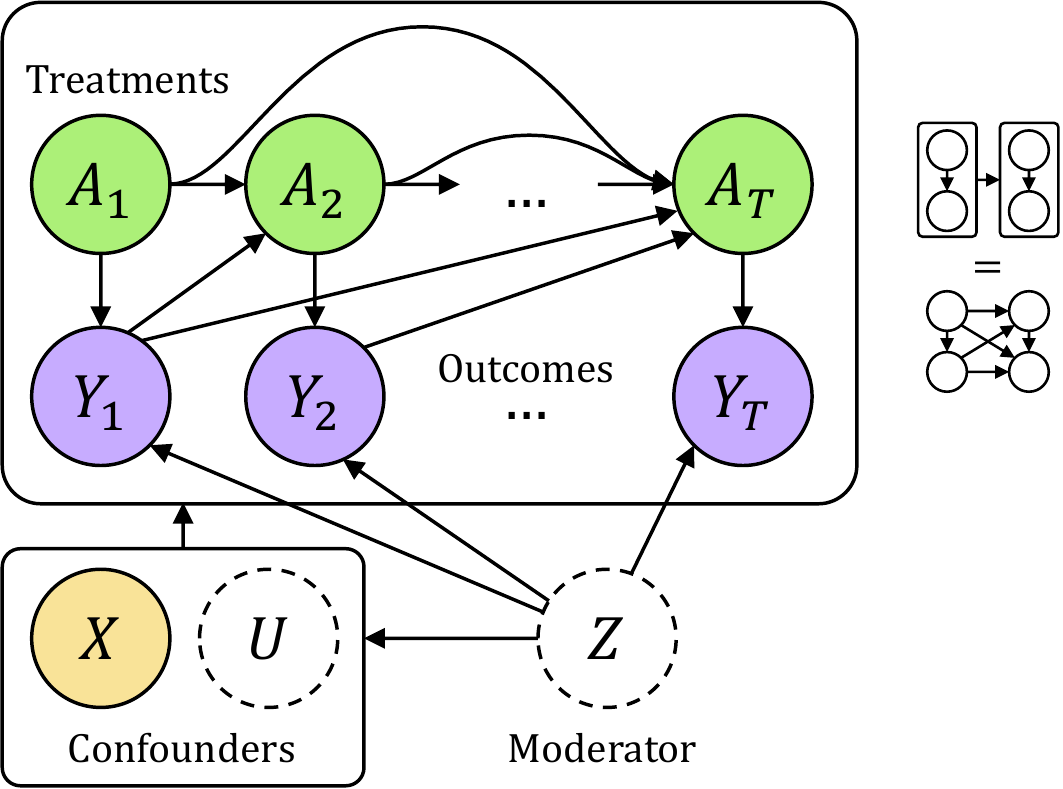}
    \caption{\label{fig:causalmodel}Assumed causal structure. Arrows between boxes indicate connections between \emph{all} variables in the boxes: $X$ is a cause of every $A_s$ and $Y_s$. Past actions are assumed not to be direct causes of future outcomes. $Z$ is a moderator of treatment effects. Dashed outlines indicate unobserved variables.}
\end{subfigure}
~
\begin{subfigure}[t]{0.48\textwidth}
    \centering
    \includegraphics[width=1.0\textwidth]{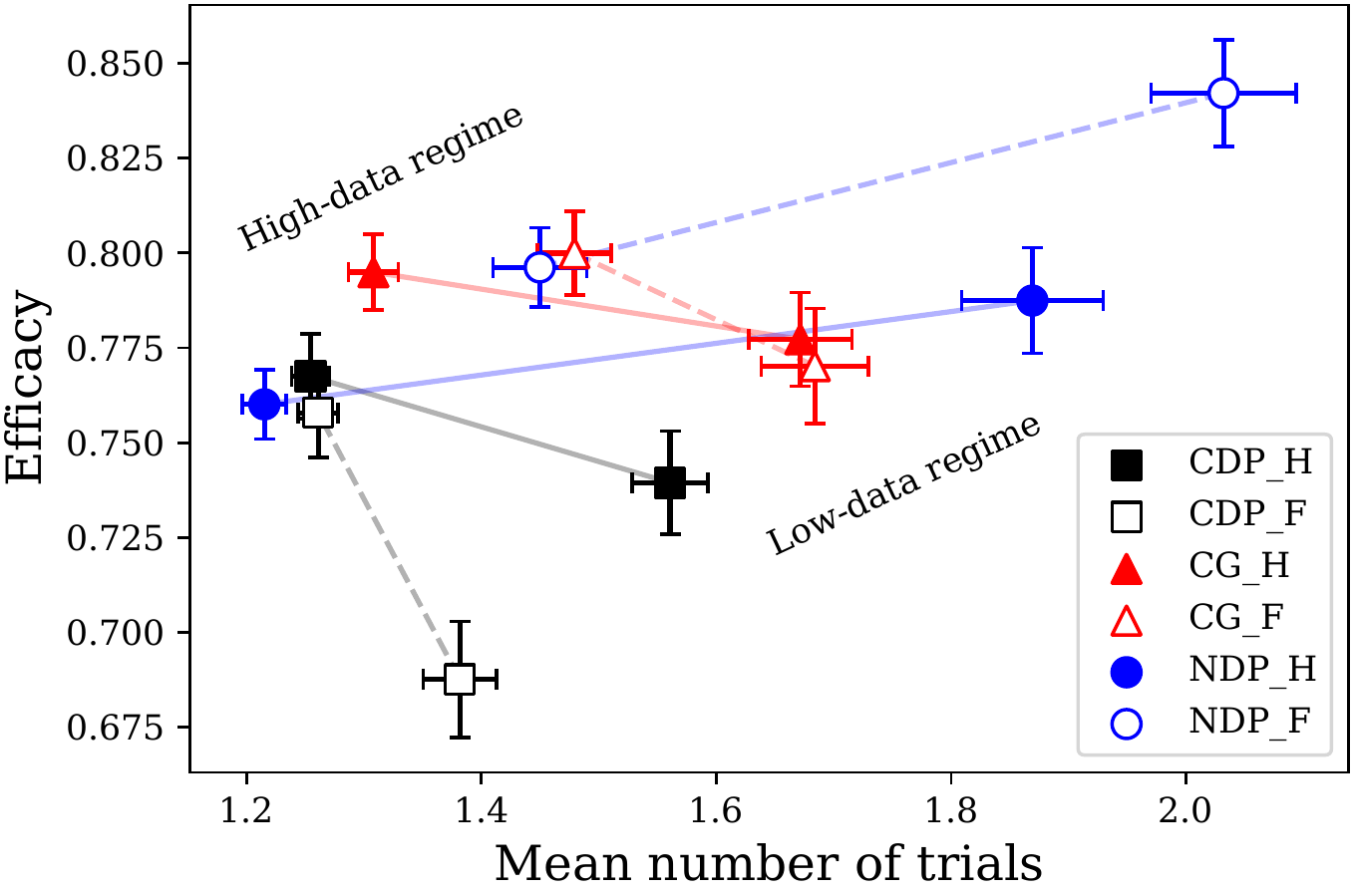}
    \caption{\label{fig:synthexp}  Efficacy (fraction of subjects for which optimal action is found) and search length for varying amounts of samples, with trade-off parameters $\delta=0.4, \epsilon=0$ (\cdp{}, \cg{}), $\lambda=0.35$ (\ndp{}). The suffix {\sc \_H} indicates historical smoothing and {\sc \_F} function approximation. Error bars indicate standard errors over 71 realizations.}
\end{subfigure}
\caption{Assumed causal structure (left) and results from synthetic experiments (right).}
\end{figure}

Ignorability follows from the backdoor criterion applied to the causal model of Figure~\ref{fig:causalmodel} when $U$ is empty~\citep{pearl2009causality}. We expand on this setting next. In contrast to conventions typically used in the literature, positivity is specified w.r.t. the considered policy class. This ensures that every action could be observed \emph{at some point} after every history $h$ that is possible under policies in $\Pi$. 
Under Assumption~\ref{asmp:stationarity} (stationarity), there is no need to try the same treatment twice, since the outcome is already determined by the first trial. We can restrict our attention to  non-repeating policies, 
$$
\Pi \subseteq \left\{ \pi \colon \cH \rightarrow \cAstop \; ; \; \pi(h) \not\in h \right\}~.
$$
Non-repeating policies such as these take the form of a decision tree of depth at most $k = |\cA|$. 

\begin{thmrem}[Assumptions~1--4 in practice] Only the positivity assumption may be verified empirically; stationarity, consistency and ignorability must be justified by domain knowledge. Readers experienced with causal estimation will be familiar with the process of establishing ignorability and consistency through graphical arguments or reasoning about statistical independences. Stationarity is more specific to our setting and without it, the notion of a near-optimal action is not well-defined---the best action could change with time. This phenomenon occurs is settings where outcomes naturally increase or decrease over time, irrespective of interventions. For example, the cognitive function of patients with Alzheimer's disease tends to decrease steadily over time~\citep{arevalo2015mini}. As a result, measures of cognitive function $Y_t(a)$ for patients on a medication $a$ will be different depending on the stage $t$ of progression that the patient is in. As a rule-of-thumb, stationarity is better justified over small time-frames or for more stable conditions.  
\label{rem:assumptions}
\end{thmrem}

%
%
\subsection{Identification without unmeasured confounders}
Our stopping criterion $\rho(h)$ is an interventional quantity which represents the probability that an unused action \emph{would be} preferable to previously tried ones. In general, this is not equal to the rate at which such an action was preferable in observed data. Nevertheless, we prove that $\rho(h)$ is identifiable from observational data in the case that $U$ does not exist (ignorability holds w.r.t. $H$). First, the following lemma shows that the \emph{order} of history does not influence the probability of future outcomes.
\begin{thmlem}%
\label{lem:historyorder}%
Let $\cI$ be a permutation of $(1, ..., s)$. Under stationarity, for all $\ola \in \olcA_s$ and  $b\not\in\ola$,
\begin{equation}
p(Y(b) \mid X, \olA_s = \ola, \olY_s = \oly) = p(Y(b) \mid X, \olA_s=(a_{\cI(1)}, ..., a_{\cI(s)}), \olY_s=(y_{\cI(1)}, ..., y_{\cI(s)})) 
\end{equation}
\end{thmlem}
Lemma~\ref{lem:historyorder} is proven in Appendix~\ref{app:stationarity}. As a consequence, we may treat two histories with the same events in different order as equivalent when estimating $p(Y(a) \mid H)$.

We can now state the following result about identification of the near-optimality constraint of \eqref{eq:mainprob}.

\begin{thmthm}
Under Assumptions~\ref{asmp:consistency}--\ref{asmp:ignorability}, the stopping criterion $\rho(h)$ in \eqref{eq:nearopt} is identifiable from the observational distribution $p(X, T, \olA, \olY)$. For any time step $s$ with history $h_s$, let $h(\cI)_s = (x, a_{\cI(1)}, ..., a_{\cI(s)}, y_{\cI(1)}, ..., y_{\cI(s)})$ be an arbitrary permutation of $h_s$. Then, for any sequence of untried actions $\ola_{s+1:k} = (a_{s+1}, ..., a_k) \in \cS(\cA_{-{h_s}})$ with $h(\cI)_r$ the (hypothetical) continued history at time $r>s$ corresponding to $\ola_{s+1:k}$ and $\oly_{s+1:k}$, and with $\mu(h_s) = \max_{(a,y) \in h_s} y$, 
\begin{align}
\rho(h_s) =  \sum_{\oly_{s+1:k} \in \cY^{k-s}}\mathds{1}\left[\max(\oly) > \mu(h_s) + \epsilon  \right] \prod_{r={s+1}}^k p(Y_r=y_r \mid A_r=a_r, H_r=h(\cI)_{r-1})~. \label{eq:main_estimator}
\end{align}
\label{thm:identifiability}
\end{thmthm}
A proof of Theorem~\ref{thm:identifiability} is given in Appendix~\ref{app:identifiability}. Equation \eqref{eq:main_estimator} gives a concrete means to estimate $\rho(h)$ from observational data by constructing a model of $p(Y_s \mid A_s, H_{s-1}=h)$. Due to Assumption~\ref{asmp:stationarity} (stationarity), this model can be invariant to permutations of $h$. 
Another important consequence of this result is that, because Theorem~\ref{thm:identifiability} holds for any future sequence of actions,  \eqref{eq:main_estimator} holds also over any convex combination for different future action sequences, such as the expectation over the empirical distribution. Using likely sequences under the behavior policy will lead to lower-variance estimates.

%
%
\label{sec:ignorable}%
\begin{thmrem}In the fully discrete case, we may estimate $p(Y \mid A, H)$ using a probability table, and we do so in some experiments in Section~\ref{sec:experiments}. However, this becomes increasingly difficult for both statistical and computational reasons when $\cA$ and $\cY$ grow larger or when any of the variables are continuous. 
The permutation invariance given by Theorem~\ref{thm:identifiability} provides some relief but, nevertheless, the number of possible combinations (histories) grows exponentially with the number of actions. As a result, it is very probable that certain pairs of histories and actions $(h,a)$ are never observed in practical applications.
We consider two remedies to this. In Appendix~\ref{app:smoothing}, we give methods for leveraging observations of similar histories $h'\approx h$ in the estimation of $p(Y \mid H=h, A)$, one based on historical kernel-smoothing in the tabular case, and one based on function approximation. These are compared empirically in Section~\ref{sec:experiments}. In Section~\ref{sec:greedy}, we give bounds to use in place of the probability of unobserved potential outcomes which further mitigate the curse of dimensionality. 
\end{thmrem}

%
%
\subsection{Accounting for unobserved confounders} 
\label{sec:nonignorable}%
If Assumption~\ref{asmp:ignorability} (ignorability) does not hold with respect to observed variables, the stopping criterion $\rho(h_s)$ may not be identified from observational data without further assumptions. A natural relaxation of ignorability is that the same condition holds w.r.t. an expanded adjustment set $(H_{s}, U)$, where $U \in \cU$ is an unobserved set of variables. This is the case in our assumed causal model, see Figure
~\ref{fig:causalmodel}. We require additionally that $U$ has bounded influence on treatment propensity. For all $u \in \cU, h\in \cH$, with $s = |h|$ and $\ola \in \cS(\cA_{-h})$, assume that there is a sensitivity parameter, $\alpha \geq 1$, such that
\begin{equation}\label{eq:sensitivity}
 \frac{1}{\alpha} \leq  \frac{\Pr[\olA_{s+1:k} = \ola \mid H_s=h]}{\Pr[\olA_{s+1:k} = \ola \mid U=u, H_s=h]} \leq \alpha~,
\end{equation}
where $\olA_{s+1:k}$ is defined as in Assumption~\ref{asmp:positivity}. Like ignorability, this assumption must be justified from external knowledge since $U$ is unobserved. We arrive at the following result.
\begin{thmthm} 
Assume that \eqref{eq:sensitivity} and  Assumptions~\ref{asmp:consistency}--\ref{asmp:ignorability} hold with respect to $(H_s, U)$ for all $s \in [k] $ with sensitivity parameter $\alpha\geq 1$. 
Then, for any $h \in \cH_s$, $\ola \in \cS(\cA_{-h})$ and $\nu = \mu(h)+\epsilon$, we have
\begin{align*}
& \Pr[\max_{r=s}^k Y_r > \nu \mid \olA_{s+1:k} = \ola, H_s = h] \leq \frac{\delta}{\alpha} 
\implies \rho(h) = \Pr[\max_{a \in \ola} Y(a) > \nu \mid H_s = h] \leq \delta~. 
\end{align*}%
\label{thm:nonignorable}%
\end{thmthm}
A proof of Theorem~\ref{thm:nonignorable} is given in Appendix~\ref{app:nonignorable}.
To achieve near-optimality with confidence level of $\delta$ in the presence of unobserved confounding with propensity influence $\alpha$, we must require a confidence level of at most $\delta/\alpha$. Unlike classical approaches to sensitivity analysis, as well as more recent results~\citep{kallus2018confounding}, this argument does not rely on importance (propensity) weighting.  

%
%
\section{Policy optimization}
\label{sec:optimization}
We give two algorithms for policy optimization under the assumption that a model of the stopping criterion $\rho(h)$ is known. As noted previously, this problem is NP-hard due to the exponentially increasing number of possible histories~\citep{rivest1987learning}. Nevertheless, for moderate numbers of actions, we may solve \eqref{eq:mainprob} exactly using dynamic programming, as shown next. Then we propose a greedy approximation algorithm and discuss model-free reinforcement learning as  alternatives.

\subsection{Exact solutions with dynamic programming}
\label{sec:cdp}
Let $X, A, Y$ be discrete. For sufficiently small numbers of actions, we can solve \eqref{eq:mainprob} exactly in this setting. 
Let $h' = h \cup \{(a,y)\}$ denote the history where $(a,y)$ follows $h$ and recall the convention $\max_{a \in \emptyset} Y(a) = -\infty$. Now define $Q$ to be the expected cumulative return---see e.g., \citet{sutton1998introduction} for an introduction---of taking action $a$ in a state with history $h \in \cH$, 
\begin{equation}
\begin{aligned}
& Q(h, a) = r(h,a) + \mathds{1}[a\neq\stop{}]\sum_{y \in \cY} p(Y(a)=y \mid h) \max_{a'\in \cAstop}Q(h \cup \{(a, y)\},a')~,
\end{aligned}%
\label{eq:qfunc}%
\end{equation}
where $r(h,a)$ is a reward function defined below. The value function $V$ at a history $h$ is defined in the usual way, $V(h) = \max_{a} Q(h, a)$.
To satisfy the near-optimality constraint of \eqref{eq:mainprob}, we use an estimate of the function $\rho(h)$, see \eqref{eq:nearopt}, to define $\gamma_{\epsilon, \delta, \alpha}(h) \coloneqq \mathds{1}[\rho(h) < \delta/\alpha ]$ 
for parameters $\epsilon, \delta \geq 0$, $\alpha \geq 1$. The function $\gamma_{\epsilon, \delta, \alpha}(h)$ represents whether an $\epsilon, \delta/\alpha$-optimum has been found. We define
\begin{equation}
r_{\epsilon, \delta, \alpha}(h,a) = \left\{
\begin{array}{ll}
-\infty, & \text{ if } a=\stop{}, \gamma_{\epsilon, \delta, \alpha}(h)=0 \\
0, & \text{ if } a=\stop{}, \gamma_{\epsilon, \delta, \alpha}(h)=1 \\
-1, & \text{ if } a\neq\stop{}%
\end{array}
\right.
~.%
\label{eq:rewardfunc}%
\end{equation}
With this, given a model of $p(Y_s(a)\mid H_{s-1}, A_s)$, the $Q$-function of \eqref{eq:qfunc} may be computed using dynamic programming, analogous to the standard algorithm for discrete-state reinforcement learning.
\begin{thmthm}
Recall that $H_0 = (X, \emptyset, \emptyset)$. The policy maximizing \eqref{eq:qfunc}, $\pi(h) = \argmax_{a} Q(h,a)$,
with reward given by \eqref{eq:rewardfunc}
is an optimal solution to~\eqref{eq:mainprob} with objective $\E_{p_\pi}[T] = \E_X[-V(H_0)]$.%
\label{thm:correctness}%
\end{thmthm}
Theorem~\ref{thm:correctness} follows from Bellman optimality and the definition of $r$ in \eqref{eq:rewardfunc}, see Appendix~\ref{app:correctness}. 

%
%
\subsection{A greedy approximation algorithm}
\label{sec:greedy}
We propose a greedy policy as an approximate solution to \eqref{eq:mainprob} in high-dimensional settings where exact solutions are infeasible to compute. 
We then discuss sub-optimality and approximation ratios of greedy algorithms. First, consider the greedy policy $\pi_G$, which chooses the treatment with the highest probability of finding a best-so-far outcome, weighted by its value, according to  
\begin{equation}
f(h,a) = \E[\mathds{1}[Y(a) > \max_{(\cdot,y)\in h} y] Y(a) \mid H_s=h]    
\end{equation}
until the stopping criterion is satisfied,
\begin{equation}
\pi_G(h) \coloneqq \left\{
\begin{array}{ll}
\stop{}, & \gamma_{\epsilon, \delta, \alpha}(h) = 1\\
\argmax_{a \not\in h} f(h, a), &  \text{otherwise}
\end{array}
\right.
~,
\label{eq:greedy}
\end{equation}
where $\gamma_{\epsilon, \delta, \alpha}$ is defined as in Section~\ref{sec:cdp}.
While using $\pi_G$ avoids solving the costly dynamic programming problem of the previous section, it still requires evaluation of $\gamma(h)$. Even for short histories, $|h| \approx 1$, computing $\gamma(h)$ involves modeling the distribution of maximum-length sequences over potentially $|\cY|^{|\cA|}$ configurations. To increase efficiency, we bound the stopping statistic $\rho$, and approximate $\gamma$, using conditional distributions of the potential outcome of single actions. 
\begin{equation}\label{eq:bounds}
\rho(h) \coloneqq p\left(\max_{a \not\in h} Y(a) > \mu(h) + \epsilon \mid h \right)
\leq \sum_{a \not\in h} p\left(Y(a) > \mu(h) + \epsilon \mid h \right)~.%
\end{equation}
A proof is given in Appendix~\ref{app:bound}. Using the upper bound in place of $\rho(h)$ leads to a feasible solution of \eqref{eq:mainprob} with more conservative stopping behavior and better outcomes but worse expected search time. In the case $\delta=0$, the exact statistic and the upper bound lead to identical policies. Representing the upper bound as a function of all possible histories still requires exponential space in the worst case, but only a small subset of histories will be observed for policies that terminate early.
We use the bound on $\rho(h)$ in experiments with both dynamic programming and greedy policies in Section~\ref{sec:experiments}. The general problem of learning bounds on  potential outcomes was studied by~\citep{makar2020estimation}.

\begin{thmex}
\label{ex:toy}
In the following example, the greedy policy does identify a near-optimal action after the smallest expected number of trials, for $\delta=0, \epsilon=0$. Let $X=0$, $Y \in \{0,1\}$ and $A \in \{1, 2, 3\}$, $Z = \{1, ..., 4\}, p(Z) = [0.20, 0.15, 0.20, 0.45]^T$ and let $C$ be the matrix with elements $c_{ij}$ such that
$$
p(Y(j) = 1 \mid Z=i) = c_{ij}, \text{ with } \; C^\top =
\begin{bmatrix}
1 & 0 & 1 & 0 \\
0 & 1 & 0 & 1 \\
1 & 0 & 0 & 1 
\end{bmatrix}.
$$
In this scenario, $p(Y(\cdot)=1) = [0.4, 0.6, 0.65]^\top$. The greedy strategy would thus start with $\pi(\emptyset) = 3$, followed by $\pi((3)) = 2$ and then $\pi((3,2)) = 1$ to guarantee successful treatment. An optimal strategy is to start with $A_1=2$ and then $A_2=1$. The expected time $\E[T]$ is 1.5 under the greedy policy and 1.4
under the optimal one. The worst-case time under the greedy strategy is 3 and 2 under the optimal. 
\end{thmex}

In Appendix~\ref{app:greedy}, we show that our problem is equivalent to a variant of active learning once a model for $p(Y(a_1), ..., Y(a_k), X)$ is known. In general, it is NP-hard to obtain an approximation ratio better than a logarithmic factor of the number of possible combinations of potential outcomes~\citep{golovin2010near,chakaravarthy2007decision}. However, for instances with additional structure, e.g., through correlations induced by the moderator $Z$, this ratio may be significantly smaller than $|\cA|\log|\cY|$.

\subsection{A model-free approach} 
\label{sec:model-free}
In off-policy evaluation, it has been noted that for long-term predictions, model-free approaches may be preferable to, and suffer less bias than, their model-based counterparts~\citep{thomas2016data}. They are therefore natural baselines for solving~\eqref{eq:mainprob}. We construct such a baseline below. 

Let $\max_{(\cdot, y) \in h} y$ represent the best outcome so far in history $h$, with $s=|h|$ and let $\lambda > 0$ be a parameter trading off early termination and high outcome. Now, consider a reward function $r(h,a)$ which assigns a reward at termination equal to the best outcome found so far. A penalty $-\lambda$ is awarded for each step of the sequence until termination, a common practice for controlling sequence length in reinforcement learning, see e.g,
~\citep{pardo2018time}. Let 
\begin{equation}
r_\lambda(h,a) = \{
0, \mbox{ if } a\neq\stop{} \;\; ; \;\;
\max_{(\cdot, y) \in h} y - \lambda|h|, \mbox { if } a=\stop{}
\}~.
\end{equation}
The policy $\pi_\lambda$ which optimizes this reward, using dynamic programming as in Section~\ref{sec:cdp}, is used as a baseline in experiments in Section~\ref{sec:experiments}. %

While this approach has the advantage of not requiring a model of future outcomes, \emph{without a model, the stopping criterion $\rho(h)$ cannot be verified and the advantage of being able to specify an interpretable certainty level is lost}. This is because the trade-off parameter $\lambda$ does not have a universal interpretation---the value of $\lambda$ which achieves a given rate of near-optimality will vary between problems. In contrast, the confidence parameter $\delta$ directly represents a bound on the probability that there is a better treatment available when stopping. Additionally, in Appendix~\ref{app:non-equivalence}, we prove that there are instances of the main problem \eqref{eq:mainprob}, for a given value of $\delta$, such that no setting of $\lambda$ results in an optimal solution. 

%
%
\section{Experiments}
\label{sec:experiments}
We evaluate our proposed methods using synthetic and real-world healthcare data in terms of the quality of the best action found, and the number of trials in the search.\footnote{Implementations can be found at: \url{https://github.com/Healthy-AI/TreatmentExploration}} In particular, we study the \emph{efficacy} of policies, defined as the fraction of subject for which a near-optimal action has been found when the choice to stop trying treatments is made. Models of potential outcomes are estimated using either a table with historical smoothing (labeled with suffix \_H) or using function approximation using random forests (suffix \_F), see Appendix~\ref{app:smoothing}. Following each estimation strategy, we compare policies learned using constrained dynamic programming (\cdp{}), the constrained greedy approximation (\cg{}) and the model-free RL variant, referred to as as na\"{i}ve dynamic programming (\ndp{}), see Section~\ref{sec:optimization}. 
Establishing near-optimality is infeasible in most observational data as only a subset of actions are explored. However, as we will see, in our particular application, it may be determined exactly.
\subsection{Synthetic experiments: Effect of sample size and algorithm choice}
To investigate the effects of data set size, number of actions, dimensionality of baseline covariates and the uncertainty parameter $\delta$ on the quality of learned policies, we designed a synthetic data generating process (DGP). This DGP parameterizes probabilities of actions and outcomes as log-linear functions of a permutation-invariant vector representation of history and of $(X, Z)$, respectively. For the results here, $\cA=\{1, ..., 5\}, \cX=\{0,1\}, \cY=\{0,1,2\}, \cZ=\{0,1\}
^3$. Due to space limitations, we give the full DGP and more results of these experiments in Appendix~\ref{app:synthetic}.

We compare the effect of training set size for the different policy optimization algorithms (\cdp{}, \cg{}, \ndp{} and model estimation schemes (\_F, \_H). Here, \cdp{} and \cg{} use  $\delta=0.4, \epsilon=0$ and the upper bound of \eqref{eq:bounds} and \ndp{} $\lambda=0.35$.  We consider training sets in a low-data regime with $50$ samples and a high-data regime of $75000$, with fixed test set size of 3000 samples. Results are averaged over 71 realizations.  
In Figure~\ref{fig:synthexp}, we see that the value of all algorithms converge to comparable points in the high-data regime but vary significantly in the low-data regime. In particular, \cg{} and \cdp{} improve on both metrics as the training set grows. The time-efficacy trade-off is more sensitive to the amount of data for \ndp{} than for the other algorithms, and while additional data significantly reduces the mean number of actions taken, this comes at a small expense in terms of efficacy.  This highlights the sensitivity of the na\"{i}ve RL-based approach to the choice of reward: the scale of the parameter $\lambda$ determines a trade-off between the number of trials and efficacy, the nature of which is not known in advance. In contrast, \cdp{} and \cg{} are preferable in that $\delta$ and $\epsilon$ have explicit meaning irrespective of the sample and result in a subject-specific stopping criterion, rather than an average-case one.
    
%
%
\subsection{Optimizing search for effective antibiotics}
Antibiotics are the standard treatment for bacterial infections. However, infectious organisms can develop resistance to specific drugs~\citep{spellberg2008epidemic} and patterns in organism-drug resistance vary over time~\citep{kanjilal2018trends}. Therefore, when treating patients, it is important that an antibiotic is selected to which the organism is susceptible. For conditions like sepsis, it is critical that an effective antibiotic is found within hours of diagnosis~\citep{dellinger2013surviving}. 

As a proof-of-concept, we consider the task of selecting effective antibiotics by analyzing a cohort of intensive-care-unit (ICU) patients from the MIMIC-III database~\citep{johnson2016mimic}. We simplify the real-world task by taking effective to mean that the organism is susceptible to the antibiotic. When treating patients for infections in the ICU, it is common that microbial cultures are tested for resistance. This presents a rare opportunity for off-policy policy evaluation, as the outcomes of these tests may be used as the ground truth potential outcomes of treatment
~\citep{boominathan2020treatment}. In practice, the results of these tests are not always available at the time of treatment. For this reason, we learn models based on the test outcomes \emph{only of treatments actually given to patients}. To simplify further, we interpret concurrent treatments as sequential; their outcomes are not conflated here since they are taken from the culture tests. We stress that this task is not meant to accurately reflect clinical practice, but to serve as a benchmark based on a real-world distribution. Although a patient's condition may change as a response to treatment, bacteria typically do not develop resistance during a particular ICU stay, and so the stationarity assumption is valid.

Baseline covariates $X$ of a patient represent their age group (4 groups), whether they had infectious or skin diseases ($2\times 2$ groups), and the identity of the organism, e.g., Staphylococcus aureus. These were found to be important predictors of resistance by~\citet{ghosh2019machine}. In total, $X$ comprised 12 binary indicators. From the full set of microbial events in MIMIC-III, we restricted our study to a subset of 4 microorganisms and 6 antibiotics, selected based on overall prevalence and the rate of co-occurrence in the data. There were three distinct final outcomes of culture tests, \emph{resistant, intermediate, susceptible}, encoded as $Y = 0, 1, 2$, respectively, where higher is better. The resulting cohort restricted to patients treated using only the selected antibiotics consisted of $n=1362$ patients which had cultures tested for resistance against all antibiotics. The cohort was split randomly into a training and test set with a 70/30 ratio and experiments were repeated over five such splits. Patients treated for multiple organisms were split into different instances. A full list of variables, the selected antibiotics and organisms, and additional statistics are given in Appendix~\ref{app:antibiotics}.

We compare our learned policies to the policy used to select antibiotics in practice. However, due to censoring, e.g., from mortality, the sequence length of observed patients may not be representative of the expected number of trials used by the observed policy before an effective treatment is found. In other words, the average outcome for patients who went through $t$ treatments is a biased estimate of the value of the observed policy. Therefore, for direct comparison with current practice (``Doctor''), only the mean outcome following the first treatment point is displayed (star marker) in Figure~\ref{fig:algo_vs_doc}. For an approximate comparison with current practice, as used in multiple treatment trials, we created a baseline dubbed ``Emulated doctor''. It uses a tabular estimate of the observed policy to imitate the choices made by doctors in the dataset in terms of the history $H = (X, \olA, \olY)$, i.e., it operates on the same information as the other algorithms. We compare this to \cdp{}, \cg{} and \ndp{}, and evaluate all policies using culture tests for held-out observations. We sweep all hyperparameters uniformly over 10 values; for \cdp{}, \cg{}, $\delta \in [0,1]$, for \ndp{}\_H,  $\lambda\in [0, 0.5]$ and for \ndp\_F, $\lambda\in [0, 1]$.

In Figure~\ref{fig:algo_vs_doc}, we see that \cg{}, \cdp{} and \ndp{}, with function approximation, all learn comparable policies that are preferable to the estimated behavior policy. The mean search length was 1.26 for \cdp{} and \ndp{}, 1.28 for \cg{} and 1.38 for Emulated doctor. We see that the best treatment found after a single trial is slightly better in the raw data (star marker). This may be because more information is available to the physician than to our algorithms. The physician could (1) take into account the original value of continuous variables, such as age, instead of using age groups and (2) use more features of the patient in order to find the right treatment. Using more covariates in this instance would make the problem impractical to solve without further approximations since the table generated by the dynamic programming algorithm grows exponentially. The current variable set was restricted for this reason. In Figure~\ref{fig:func_vs_freq}, we see that across different values of $\delta, \lambda$, all algorithms achieve near-optimal efficacy (almost 1), but vary in their search time. \cdp{} is equal or preferable to \cg{}, with the model-free baseline \ndp{} achieving the worst results. A much more noticeable difference is that between policies learned using the model estimated with function approximation (suffix \_F) and those with a (smoothed) tabular representation (suffix \_H). 

\begin{figure}
    \centering
    \begin{subfigure}[t]{0.48\textwidth}
    \centering
        \includegraphics[height=0.21\textheight]{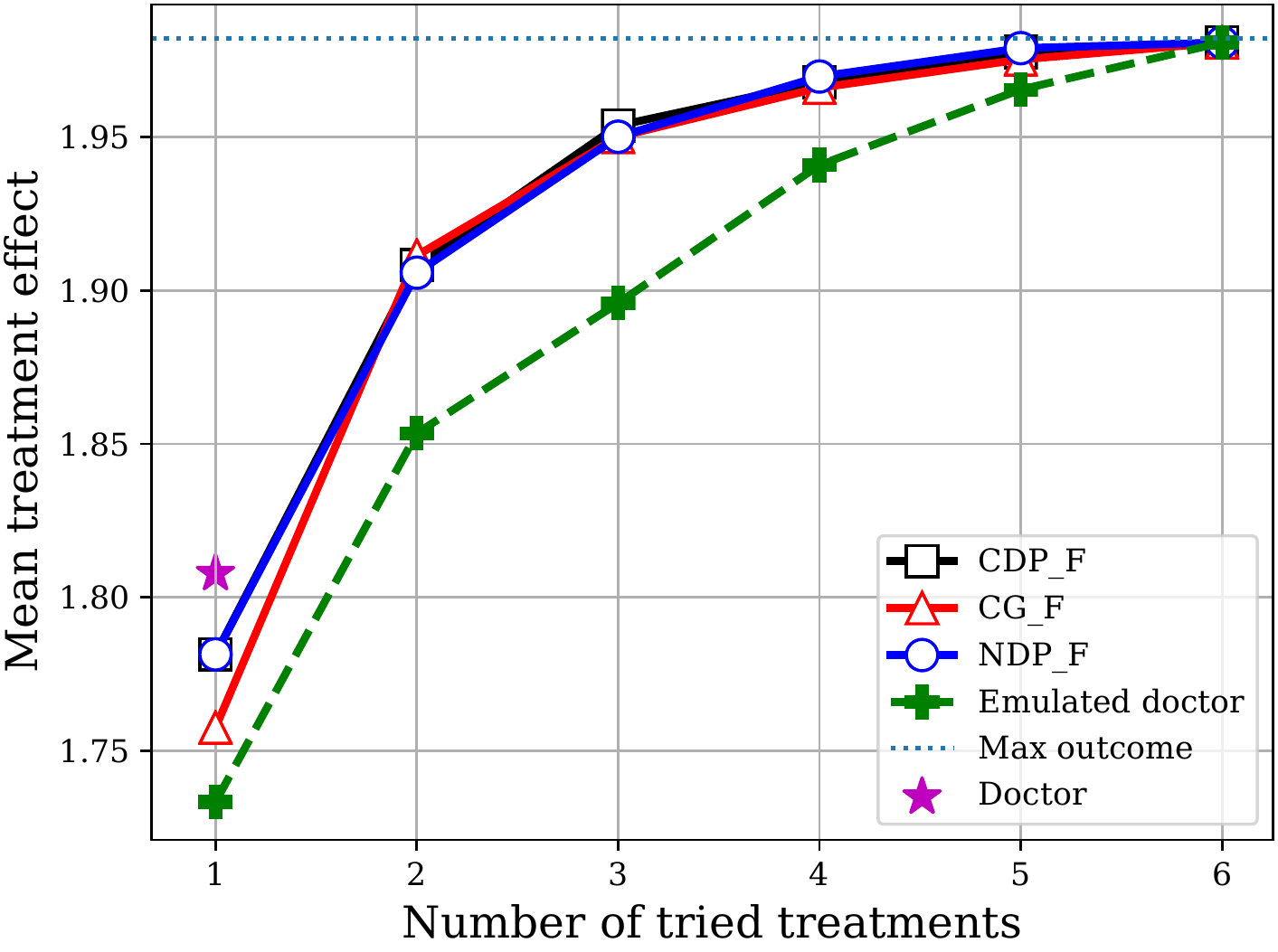}
        \caption{Mean best-at-termination or best-so-far outcome found after a given number of trials, across all subjects, for different policies. $\delta = 0, \lambda=0.35$.}
        \label{fig:algo_vs_doc}
    \end{subfigure}
    ~
    \begin{subfigure}[t]{0.48\textwidth}
        \centering
        \includegraphics[height=0.21\textheight]{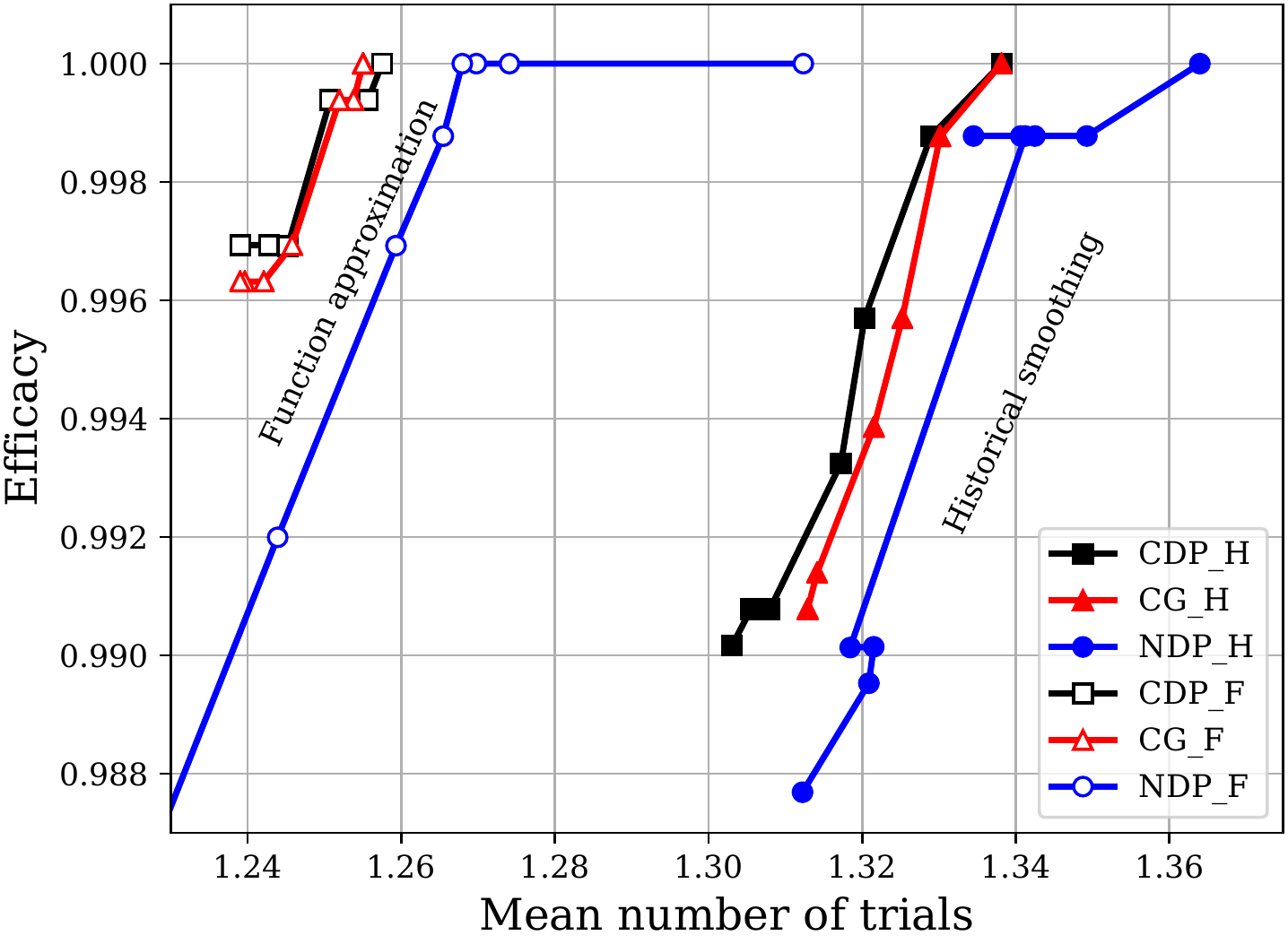}
        \caption{Efficacy of antibiotics vs the mean number of trials for different values of $\delta$ and $\lambda$ (one value per marker) and different model estimation schemes.}
        \label{fig:func_vs_freq}
    \end{subfigure}
    \caption{Results from the antibiotics experiment. Average best-found outcome of different policies across patients at different stages of the search (a) and efficacy and search time (number of treatment trials) as functions of $\delta$ (b). In plot (a), at a given number of trials, the best-so-far outcome is used for ongoing sequences, and the best-at-termination is used for terminated ones. Efficacy refers to the rate at which a near-optimal treatment is found at the given $\delta$. Suffixes \_F and \_H indicates model estimation using function approximation and historical smoothing respectively. $\epsilon = 0$.
    \label{fig:antibiotics_results}}%
\end{figure}

\section{Conclusion}
We have formalized the problem of learning to search efficiently for causally effective treatments. We have given conditions under which the problem is solvable by learning from observational data, and proposed algorithms that estimate a causal model and perform policy optimization. Our solution using constrained dynamic programming (CDP) in an exponentially large state space illustrates the associated computational difficulties and prompted our investigation of two approximations, one based on greedy search and one on model-free reinforcement learning. We found that the greedy search algorithm performed comparably to the exact solution in experiments and was less sensitive to sample size. Determining  conditions under which greedy algorithms are preferable statistically is an interesting open question. We believe that our work will have the largest impact in settings where a) the assumption of potential outcome stationarity is justified, b) even a small reduction in search time is valuable and c) a transparent trade-off between efficacy and search time is valuable in itself. 
%

%

%
%
\section*{Broader impact}
Personalized and partially automated selection of medical treatments is a long-standing goal for machine learning and statistics with the potential to improve the lives of patients and reduce the workload on physicians. This task is not without risk however, as poor decisions may fail to reduce or even increase suffering. It is important that implementations of such ideas is guided by strong domain knowledge, thorough evaluation and that checks and balances are in place. Many previous works in this field aim to identify new policies for treatment or doses with the goal of improving treatment response itself. This goal is not always feasible to achieve---some conditions are fundamentally hard to treat with available medications and procedures. In contrast, we focus on conditions where a good enough treatment would be identified by an existing policy given enough time, with the goal of reducing this search time as much as possible. The trade-off between a good outcome and time is made transparent using a model of patient outcomes and a certainty parameter. With this, we hope to contribute towards making machine learning methods more suitable for clinical implementation.

\section*{Funding disclosure}
This work was supported in part by the Wallenberg AI, Autonomous Systems and Software Program (WASP) funded by the Knut and Alice Wallenberg Foundation.

%

%
%

\bibliographystyle{humannat}
\bibliography{main}

\appendix

\setcounter{equation}{0}
\setcounter{table}{0}
\setcounter{figure}{0}
\setcounter{thmthm}{0}
\setcounter{thmprop}{0}
\setcounter{thmlem}{0}
\setcounter{thmcol}{0}
\setcounter{thmconj}{0}
\setcounter{thmasmp}{0}
\setcounter{thmcorr}{0}

\def\theequation{S\arabic{equation}}
\def\thetable{S\arabic{table}}
\def\thefigure{S\arabic{figure}}
\renewcommand{\thethmthm}{S\arabic{thmthm}}
\renewcommand{\thethmprop}{S\arabic{thmprop}}
\renewcommand{\thethmlem}{S\arabic{thmlem}}
\renewcommand{\thethmcol}{S\arabic{thmcol}}
\renewcommand{\thethmconj}{S\arabic{thmconj}}
\renewcommand{\thethmasmp}{S\arabic{thmasmp}}
\renewcommand{\thethmcorr}{S\arabic{thmcorr}}

%
%
\clearpage

\makeatletter
\newcommand{\settitle}{\@maketitle}
\title{Supplementary material for: Learning to search efficiently for causally near-optimal treatments}
\author{%
  Samuel H\aa{}kansson  \\
  University of Gothenburg\\
  \texttt{samuel.hakansson@gu.se} \\
  \And
  Viktor Lindblom  \\
  Chalmers University of Technology\\
  \texttt{viklindb@student.chalmers.se} \\
  \And
  Omer Gottesman \\
  Brown University\\
  \texttt{omer\_gottesman@brown.edu} \\
  \And
  Fredrik D. Johansson  \\
  Chalmers University of Technology\\
  \texttt{fredrik.johansson@chalmers.se} \\
}
\makeatother
\settitle

\section{Proofs of theorems}
\label{app:proofs}

\subsection{Proof of Lemma~\ref{lem:historyorder} (Stationarity)}
\label{app:stationarity}

\begin{thmlem}[Lemma~\ref{lem:historyorder} restated] 
\label{lem:app_stationarity}%
Let $\cI$ be a permutation of the sequence $(1, ..., s)$. Then, for our causal graph under Assumption~\ref{asmp:stationarity}, for $b\in \cA$,
$$
p(Y(b) \mid X, \olA_s = \ola, \olY_s = \oly) = p(Y(b) \mid X, \olA_s=(a_{\cI(1)}, ..., a_{\cI(s)}), \olY_s=(y_{\cI(1)}, ..., y_{\cI(s)})) 
$$
\end{thmlem}
\begin{proof}
Let $h=(x, (a_1, y_1), ..., (a_s, y_s))$. Let $\pi$ be a permutation of ${1, ..., s}$ and $\pi(r)$ the index assigned to $r$. We use the short-hands $p(a) = p(A=a)$, $p(A\mid b) = p(A\mid B=b)$, etc.
\begin{align*}
& p(Y(a) \mid H_s=h_s, A_s=a_s) & \text{stationarity}\\
& = \frac{p(Y_s(a), h_s,  a_s)}{p(h_s, a_s)} \\
& = \frac{\sum_z p(Y_s(a), h_s, a_s, z)}{\sum_z p(h_s, a_s, z)} & \text{prob. laws}\\
& = \frac{\sum_z  p(Y_s(a) \mid h_s, a_s, z) p(a_s \mid h_s, z) p(h_s \mid z) p(z)}{\sum_z p(a_s \mid h_s, z)p(h_s \mid z) p(z)}  & \text{expand} \\
& = \frac{\sum_z  p(Y_s(a) \mid h_s, a_s, z) p(a_s \mid h_s, z) \prod_r p(y_r \mid h_r, a_r, z)p(a_r \mid h_r, z) p(z)}{\sum_z \prod_r p(y_r \mid h_r, a_r, z)p(a_r \mid h_r, z) p(z)} & \text{expand history} \\
& = \frac{\sum_z  p(Y_s(a) \mid a_s, z) p(a_s \mid h_s) \prod_r p(y_r \mid a_r, z)p(a_r \mid h_r) p(z)}{\sum_z  p(a_s \mid h_s) \prod_r p(y_r \mid a_r, z)p(a_r \mid h_r) p(z)} 
& \text{$A_s \indep Z \mid H_s$} \\
& = \frac{\sum_z  p(Y_s(a) \mid z) \prod_r p(y_r(a_r) \mid z)p(z)}{\sum_z \prod_r p(y_r(a_r) \mid z)p(z)} 
& \text{cancel terms}\\
& = \frac{\sum_z  p(Y(a) \mid z) \prod_r p(Y_{\pi(r)}(a_r)=y_r \mid z)p(z)}{\sum_z \prod_r p(Y_{\pi(r)}(a_r)=y_r \mid z)p(z)}
& \text{stationarity} 
\end{align*}
\end{proof}
Since the last expression is invariant to $\pi$, the result follows. 

\subsection{Proof of Theorem~\ref{thm:identifiability} (Identifiability)}
\label{app:identifiability}

\begin{thmthm}[Theorem~\ref{thm:identifiability} restated]
Under Assumptions~\ref{asmp:consistency}--\ref{asmp:ignorability}, the stopping statistic $\rho(h)$ in \eqref{eq:nearopt} and $\epsilon, \delta$-optimality are identifiable from the observational distribution $p(X, T, \olA, \olY)$. In particular, for any time step $s$ with history $h_s$, let $h(\cI)_s = (x, a_{\cI(1)}, ..., a_{\cI(s)}, y_{\cI(1)}, ..., y_{\cI(s)})$ be an arbitrary permutation of $h_s$. Then, for any sequence of untried (future) actions $\ola_{s+1:k} = (a_{s+1}, ..., a_k) \in \cS(\cA_{-{h_s}})$ with $h(\cI)_r$ the continued history at time $r>s$ corresponding to $\ola_{s+1:k}$ and $\oly_{s+1:k}$,
\begin{align}
\rho(h_s) =  \sum_{\oly_{s+1:k} \in \cY^{k-s}}\mathds{1}\left[\max(\oly) > \mu(h_s) + \epsilon  \right] \prod_{r={s+1}}^k p(Y_r=y_r \mid A_r=a_r, H_r=h(\cI)_{r-1})~. \label{eq:main_estimator_app}
\end{align}
where $\mu(h_s) = \max_{(a,y)\in h_s} y$.
\label{thm:identifiability_app}
\end{thmthm}

\begin{proof}
Fix any history $h = (x, (a_1, y_1), ..., (a_s, y_s)) \in \cH$ with $s = |h|$, any time points $q, r \in [k]$, any $a\in \cA$ and let $\ola \in \cS(\cA)$ such that the subsequence $\ola_s = (a_1, ..., a_s)$ coincides with $h$. Then, by Assumption~\ref{asmp:stationarity}, we have
$$
Y_r(a) = Y_q(a) = Y(a) \;\;\mbox{ and }\;\; \max_{a \not\in h} Y(a) = \max_{r=s+1}^k Y_r(a_r)~.
$$
Below, we sum over sequences of outcomes $\oly_{s+1:k} = (y_{s+1}, ..., y_k) \in \cY^{k-s}$ and refer to the history $h_r$ for $r>s$. Here, $h_r = (x, (a_1, y_1), ..., (a_r, y_r))$ is a sequence of both observed actions and outcomes (corresponding to the sub-sequence $h_s \subseteq h_r$) and unobserved ones. By definition, we have for any sequence of actions $\ola \in S(\cA)$ according to the above, for any $\mu \in \cY$
\begin{align*}
\rho_\mu(h_s) 
& = \sum_{\oly_{s+1:k}\in \cY^{k-s}} p([Y(a_{s+1}), \ldots, Y(a_k)] = \oly_{s+1:k} \mid H_s = h_s)\mathds{1}[\max(\oly_{s+1:k}) > \mu] \\
& = \sum_{\oly_{s+1:k} \in \cY^{k-s}}\mathds{1}\left[\max(\oly_{s+1:k}) > \mu \right] \prod_{r={s+1}}^k p(Y_r(a_r)=y_r \mid H_{r-1} = h_{r-1}) \\
& = \sum_{\oly_{s+1:k} \in \cY^{k-s}}\mathds{1}\left[\max(\oly_{s+1:k}) > \mu \right] \prod_{r={s+1}}^k p(Y_r=y_r \mid A_r = a_r, H_{r-1} = h_{r-1})~.
\end{align*}
In the second step we apply Assumption~\ref{asmp:stationarity} (stationarity) and in the third Assumptions~\ref{asmp:consistency}--Assumptions~\ref{asmp:ignorability} (consistency, sequential ignorability). 
Finally, from ignorability and stationarity, we have for any permutations $h(\cI)_s$, 
$$
\rho_\mu(h_s) = \sum_{\oly_{s+1:k} \in \cY^{k-s}}\mathds{1}\left[\max(\oly_{s+1:k}) > \mu \right] \prod_{r={s+1}}^k p(Y_r=y_r \mid A_r = a_r, H_{r-1} = h(\cI)_{r-1})~.
$$
Doing so, we obtain the result in \eqref{eq:main_estimator}. In solving \eqref{eq:mainprob}, we only need to evaluate $\rho(h_s)$ for histories with positive support under $p_\pi$. Assumption~\ref{asmp:positivity} (positivity) ensures that there exists at least one permutation $\ola \in \cS(\cA_{-h_s})$ such that $p(A_{s+1:k} = \ola \mid H_s = h_s)$. This in turn implies identifiability. 
\end{proof}

\subsection{Bounds on stopping criterion}
\label{app:bound}
\begin{thmthm}
For any threshold $\mu \in \cY$ and history $h \in \cH$, we have under Assumption~\ref{asmp:stationarity},
\begin{equation}
\begin{aligned}
\underbrace{\max_{a \not\in h} \left[ p\left( Y(a) > \mu \mid h \right) \right]}_{\textnormal{Used for less conservative stopping}} 
\leq \underbrace{p\left(\max_{a \not\in h} Y(a) > \mu \mid h \right)}_{=: \; \rho_\mu(h)}
\leq \underbrace{\sum_{a \not\in h} p\left(Y(a) > \mu \mid h \right)}_{\textnormal{Used for more conservative stopping}}%
\label{eq:unionbound2}%
\end{aligned}
\end{equation}
\end{thmthm}
\begin{proof}
Let $\cA_{-h} = \{a \in \cA \colon a\not\in h\}$. We start with the upper bound. By definition 
$$
\{ \oly \in \cY^{|\cA_{-h}|} \colon \max(\oly) > \mu\} = \bigcup_{a \in \cA_{-h}} \{\oly \in \cY^{|\cA_{-h}|} \colon y_a > \mu \}
$$ 
Hence, by Boole's inequality, 
\begin{align*}
p\left(\max_{a \not\in h} Y(a) > \mu \mid H=h \right) 
& \leq \sum_{a \in \cA_{-h}} \sum_{\oly \in \cY^{|\cA_{-h}|} \colon y(a) > \mu } p\left( Y(\cA_{-h}) = \oly \mid H=h \right) \\
& = \sum_{a \in \cA_{-h}} p\left(Y(a) > \mu \mid h \right)~.
\end{align*}

For the lower bound, the argument is equally straight-forward.
\begin{align*}
p\left(\max_{a \not\in h} Y(a) > \mu \mid H=h \right) 
& = \sum_{\oly \colon \max(\oly)>\mu} p\left(Y(\cA_{-h}) = \oly \mid H=h \right) \\
& \geq \max_{a \in \cA_{-h}} \sum_{\oly \colon y(a) >\mu} p\left(Y(\cA_{-h}) = \oly \mid H=h \right) \\
& = \max_{a \in \cA_{-h}} p\left(Y(a) > \mu \mid H=h \right)~.
\end{align*}
\end{proof}
%
%

\subsection{Proof of Theorem~\ref{thm:nonignorable}}
\label{app:nonignorable}
We restate the following assumption and Theorem~\ref{thm:nonignorable} for convenience. 
\begin{thmasmp}
A random variable $U$ has $\alpha$-bounded propensity sensitivity relative to $H$ if for all $u \in \cU, h\in \cH$, with $s = |h|$ and $\ola \in \cA^{k-s}$, for some $\alpha \geq 1$, with $\olA_{s+1:k} = (A_{s+1}, ..., A_k)$,
$$
 \frac{1}{\alpha} \leq  \frac{\Pr[\olA_{s+1:k}  = \ola \mid H=h]}{\Pr[\olA_{s+1:k}  = \ola \mid U=u, H=h]} \leq \alpha~.
$$%
\label{asmp:app_sens}%
\end{thmasmp}%

\begin{thmthm}[Theorem~\ref{thm:nonignorable} restated] 
Given is that Assumption~\ref{asmp:app_sens} (bounded propensity) holds for $H, U$ with sensitivity parameter $\alpha\geq 1$ and Assumption~\ref{asmp:ignorability} (ignorability) holds for all $s \in [k] $ w.r.t. confounders $(H_s, U)$. 
Let $Y_r$ be the (hypothetical) outcome of treatment $A_r$ at time $r=s+1, ..., k$.
Then, for any history $h \in \cH_s$ and the set of treatments $\ola = \cA \setminus \cA(h)$, it holds that
\begin{align*}
& \Pr[\max_{r=s+1}^k Y_r > \mu \mid \olA_{s+1:k}  = \ola, H_s = h] \leq \frac{\delta}{\alpha} 
\implies \Pr[\max_{a \in \ola} Y(a) > \mu \mid H_s = h] \leq \delta~.
\end{align*}%
\label{thm:app_nonignorable}%
\end{thmthm}

\begin{proof}
We have by definition, where $\oly > \mu$ applies element-wise, 
\begin{align*}
& \Pr[\max_{a\in \ola} Y(a) > \mu \mid H=h] = \sum_{\oly : \oly > \mu} \Pr[\olY(\ola) = \oly \mid H=h] \\
& \Pr[\max_{i} Y_i > \mu \mid \olA = \ola, H=h] = \sum_{\oly : \oly > \mu} \Pr[\olY = \oly \mid \olA=\ola, H=h]
\end{align*}
Then, marginalizing over the unobserved confounder $U$ and conditioning on $H$, 
\begin{align*}
\Pr[\max_{a\in \ola} Y(a) > \mu \mid H=h] 
& = \sum_{\substack{\oly : \oly > \mu \\ u \in \cU}} \Pr[\olY(\ola) = \oly \mid h, U=u]p(U=u \mid h) \\
& = \sum_{\substack{\oly : \oly > \mu \\ u \in \cU}} \Pr[\olY = \oly \mid h, u, \ola]p(u \mid h)
\end{align*}
where the last equality follows from ignorability w.r.t. $H, U$. Applying the same steps to $\Pr[\max_{i} Y_i > \mu \mid \olA = \ola]$, we get 
$$
\Pr[\max_{i} Y_i > \mu \mid h, \olA = \ola] = \sum_{\substack{\oly : \oly > \mu \\ u \in \cU}} \Pr[\olY = \oly \mid h, u, \ola]p(u \mid \ola, h)
$$
We find that 
{
\allowdisplaybreaks
\begin{align*}
& \Pr[\max_{a\in \ola} Y(a) > \mu \mid h] - \Pr[\max_{i} Y_i > \mu \mid h, \olA = \ola]  \\
& = \sum_{\substack{\oly : \oly > \mu \\ u \in \cU}} \Pr[\oly \mid h, u, \ola]\left(p(u \mid h) - p(u \mid \ola, h) \right)  \\
& = \sum_{\substack{\oly : \oly > \mu \\ u \in \cU}} \Pr[\oly \mid h, u, \ola]p(u \mid \ola, h)\left(\frac{p(u \mid h)}{p(u \mid \ola, h)} - 1 \right)   = (*)\\
\end{align*}
By Bayes rule, we have
$$
\frac{p(u \mid h)}{p(u \mid \ola, h)} = \frac{p(u \mid h)p(\ola \mid h)}{p(\ola \mid u, h)p(u \mid h)} = \frac{p(\ola \mid h)}{p(\ola \mid u, h)}
$$
and so, 
\begin{align*}
(*) & = \sum_{\substack{\oly : \oly > \mu \\ u \in \cU}} \Pr[\oly \mid h, u, \ola]p(u \mid \ola, h)\left(\frac{p(\ola \mid h)}{p(\ola \mid u, h)} -1 \right) 
\end{align*}
}
The result follows immediately from our Assumption~\ref{asmp:app_sens}, that $\frac{1}{\alpha} \leq \frac{p(\ola \mid h)}{p(\ola \mid u, h)} \leq \alpha$. In fact, only the upper bound is needed. 
\end{proof}

%
%
\subsection{Proof of Theorem~\ref{thm:correctness} (Correctness of dynamic programming)}
\label{app:correctness}

\begin{thmthm}[Theorem~\ref{thm:correctness} restated] 
Recall that $H_0 = (X, \emptyset, \emptyset)$. The policy maximizing \eqref{eq:qfunc}, $\pi(h) = \argmax_{a} Q(h,a)$,
is an optimal solution to~\eqref{eq:mainprob} and its expected search time is $\E[T] = \E_X[-V(H_0)]$.%
\end{thmthm}
\begin{proofsketch}
Recall that 
\begin{equation}%
\gamma_{\epsilon, \delta, \alpha}(h) \coloneqq \mathds{1}[\Pr[\max_{a' \not\in h} Y(a') > \mu(h) + \epsilon \mid H=h ] < \delta/\alpha ]
 \end{equation}
\begin{equation}
\begin{aligned}
& Q(h, a) = r(h,a) + \mathds{1}[a\neq\stop{}]\sum_{y \in \cY} p(Y(a)=y \mid h) \max_{a'\in \cAstop}Q(h \cup \{(a, y)\},a')~,
\end{aligned}%
\label{eq:qfunc_app}%
\end{equation}
\begin{equation}
r_{\epsilon, \delta, \alpha}(h,a) = \left\{
\begin{array}{ll}
-\infty, & \text{ if } a=\stop{}, \gamma_{\epsilon, \delta, \alpha}(h)=0 \\
0, & \text{ if } a=\stop{}, \gamma_{\epsilon, \delta, \alpha}(h)=1 \\
-1, & \text{ if } a\neq\stop{}
\end{array}
\right.
~.%
\label{eq:rewardfunc_app}%
\end{equation}
and $V(h) = \max_{a,c} Q(h, a)$. 

By definition, any policy that achieves a finite expected reward $\E_{H_0}[V(H_0)]$ satisfies the stopping criterion, and is therefore a \emph{feasible solution} to \eqref{eq:mainprob}. Furthermore, any time search is terminated ($a=\stop{}$ or $\cA_{-h}=\emptyset$), the expected sum of rewards for a sequence is equal to minus the number of steps spent until the sequence terminates. The sequence is optimal if it terminates as soon as an $\epsilon, \delta$-optimal treatment is found. Thus, a policy with finite expected return that maximizes $V(H_0) = \max_{a} Q(H_0, a)$ is an optimally efficient search policy for effective treatments. 
\end{proofsketch}

\subsection{Approximation ratio of greedy algorithms}
\label{app:greedy}

The active learning problem concerns identification of a hypothesis $g \in \cG$ by iteratively performing tests suggested by a policy~\citep{guillory2009average}.  The problem then amounts to finding a policy $\pi$ which selects tests $\olA = A_1, ..., A_T$, the results $Y(A_1), ..., Y(A_T)$ of which identify $g$ with probability 1, $p(G=g \mid Y(A_1), ..., Y(A_T))=1$. We consider now the case were a prior distribution $p(G, Y(1), ..., Y(k))$ is known, as studied by~\citep{guillory2009average}. A sequence of tests $\olA$ which identifies $g$ is associated with a cost $c(\olA, G)$, and the objective is to find $\pi$ which minimizes the expected cost over $p$,
$$
c(\pi) = \E_{G, A\sim \pi}[c(\olA, G)]~.    
$$
We have the following result from the literature. 
\begin{thmthm}[Adapted from Theorem 4 of \citep{kosaraju1999optimal}] There exists a greedy policy $\pi$ such that for any $p$ such that ${Y(a) : a \in \cA}$ are deterministic given $G$, 
$$
c(\pi) \leq c(\pi^*)O(\log |\cG|)
$$
where $\pi^* = \argmin_{\pi'} c(\pi')$~.
\end{thmthm}
This bound is matched by a lower bound by~\citep{chakaravarthy2007decision} which states that it is NP-hard to achieve an approximation ratio better than $o(\log |\cG|)$.

In the setting with $\delta=0$, our problem may posed as active learning where the hypothesis corresponds to the maximum value of potential outcomes, $G = \min\{g \in \cY \colon p(\max_{a} Y(a) > g) \leq 0\}$. Once this quantity is identified, the stopping criterion may be determined immediately. However, under this hypothesis, ${Y(a)}$ are not deterministic given $G$ and the results above do not apply. \citet{golovin2010near} study the noisy case under the assumption that non-determinism in $Y(a)$ is controlled by a noise variable $\Theta$, i.e., that $\olY(\cA) = f(G, \Theta)$ for some deterministic function $f$.

\begin{thmthm}[Adapted from Theorem 3 in \citep{golovin2010near} with uniform costs] Fix hypotheses $\cG$, tests $\cA$ and outcomes in $\cY$, Fix a prior $p(G, \Theta)$ and a function $f \colon G \times \textnormal{supp}(\Theta) \rightarrow \cY^{|\cA|}$ which define the
probabilistic noise model. Let $c(\pi)$ denote the expected cost of $\pi$ incurs to identify which equivalence class $G$ the outcome vector $\olY(A_T)$ belongs to. Let $\pi^*$ denote the policy
minimizing $c(\cdot)$, and let $\pi$ denote the adaptive policy implemented by the greedy algorithm EC2. Then, 
$$
c(\pi) \leq c(\pi^*)O(\log |\cA| + \log |\textnormal{supp}(\Theta)|)~.
$$
\end{thmthm}

In the case that all combinations of outcomes are feasible, $\log|\textnormal{supp}(\Theta)| = |\cA|\log |\cY|$ and the bound above is vacuous, since a trivial bound on the search time is $|\cA|$. When there is structure in potential outcomes, $\textnormal{supp}(\Theta)$ may be much smaller. For example, if the moderating variable $Z$ controls all uncertainty in $Y(a)$, given X, the bound reduces to $O(\log |\cZ_X|)$ where $\cZ_X = \{z \in \cZ \colon p(Z\mid X)>0\}$, which may be significantly smaller than $|\cA|\log |\cY|$.

\subsection{Model-free RL and CDP are not equivalent}
\label{app:non-equivalence}

\def\rmf{r^{\text{model-free}}}
Let $\max_{(\cdot, y) \in h} y$ represent the best outcome so far at history $h$, with $s=|h|$ and $\lambda > 0$ a parameter trading off early termination and high outcome. Now, consider the reward function $\rmf_\lambda(h,a)$ following history $h \in \cH$ defined below. 
\begin{equation}
\rmf_\lambda(h,a) = \left\{
\begin{array}{ll}
0, & a\neq \stop{} \\
\max_{(\cdot, y) \in h} y - \lambda|h|, & a = \stop{}.
\end{array}
\right.
\label{eq:modelfree_reward_app}
\end{equation}
and the policy maximizing the expected sum of rewards
\begin{equation}
\pi^{*, \text{model-free}}_\lambda = \argmax_{\pi} \E_{h,a\sim \pi}\left[\sum_{s=1}^k \rmf_\lambda(h_s, a_s) \right]~.
\label{eq:modelfree_problem_app}
\end{equation}
Now consider the greedy policy maximizing the Q-function defined by 
\begin{equation}
Q(h, a) = \E_{h' \mid h,a}[\rmf_\lambda(h,a) + \max_{a' \in \cA_{-h}\cup \{\stop\}} Q(h', a') \mid H_s=h, A_s=a]~.
\label{eq:model_free_q_app}
\end{equation}
For readers familiar with reinforcement learning, it is easy to see that policy maximizing $Q$ defined above also maximizes the expected sum of rewards given by \eqref{eq:modelfree_reward_app}.  Below, we prove that this algorithm does not in general solve \eqref{eq:mainprob}.

\begin{thmthm}\label{thm:modelfree}
There are instances of \eqref{eq:mainprob} (main problem), specified by a distribution $p$ and parameters $\epsilon, \delta$, such that the solutions to \eqref{eq:mainprob} and \eqref{eq:modelfree_problem_app} are distinct for every choice of $\lambda>0$.
\end{thmthm}
\begin{proof}
Consider a context-less setting with two actions $\cA = \{a, b\}$ with the following potential outcomes: $p(Y(a)=1.0) = 1/2, p(Y(a)=0.5) = 1/2$ and $p(Y(b)=0.5+\epsilon) = 1$. In this scenario, having observed nothing, the probability that action $b$ yields a higher outcome than $a$ is 1/2. Hence, for $\delta=0.5$, CDP always prefers to start with action $b$ and end immediately. Now, consider NDL, which minimizes the expected return with the reward function,
$$
r(h, a) = \left\{
\begin{array}{ll}
0, & a\neq \stop{} \\
\max_{(\cdot, y) \in h} y - \lambda|h|, & a = \stop{}
\end{array}
\right.
$$
where $s$ indicates the stop action and $\max_{(\cdot, y) \in h} y$ represents the best outcome so far at history $h$ and $\lambda > 0$. The Q-function is in \eqref{eq:model_free_q_app}.
NDP computes this recursively and uses the policy which maximizes it. Under the version of this problem with $\epsilon < 0.25$, we can show that there is no $\lambda>0$ such that $Q(\emptyset, b) > Q(\emptyset, a)$.  We give the map of $Q$ below under this assumption.

\begin{table}[t]
    \centering
    \begin{tabular}{cccccc}
        \multicolumn{4}{c}{$h$} & $a$  & $Q(h, a)$ \\
         $A_1$ & $Y_1$ & $A_2$ & $Y_2$ & &  \\ \midrule
         a & 1.0 & -- & -- & \stop{} & $1.0-\lambda$ \\
         a & 0.5 & -- & -- & \stop{} & $0.5-\lambda$ \\
         b & $0.5+\epsilon$ & -- & -- & \stop{} & $0.5+\epsilon-\lambda$ \\
         a & 1.0 & b & $0.5+\epsilon$ & \stop{} & $1.0-2\lambda$ \\
         a & 0.5 & b & $0.5+\epsilon$ & \stop{} & $0.5+\epsilon-2\lambda$ \\
         a & 1.0 & -- & -- & b & $1.0-2\lambda$ \\
         a & 0.5 & -- & -- & b &
         $0.5+\epsilon-2\lambda$ \\
         b & $0.5+\epsilon$ & -- & -- & a & $\frac{(1.0-2\lambda) + (0.5+\epsilon-2\lambda)}{2}$ \\
         -- & -- & -- & -- & a &
         $\frac{(1.0-\lambda) + \max(0.5-\lambda, 0.5+\epsilon-2\lambda)}{2}$ \\
         -- & -- & -- & -- & b &
         $\max(0.5+\epsilon-\lambda, \frac{(1.0-2\lambda) + (0.5+\epsilon-2\lambda)}{2})$ \\
    \end{tabular}
    \label{tab:my_label}
\end{table}

For $\lambda > \epsilon$, $Q(\emptyset, a) = 0.75 - \lambda$ and $Q(\emptyset, b) = \max(0.5+\epsilon - \lambda, 0.75+\epsilon/2-2\lambda) < Q(\emptyset, a)$. For $0 <\lambda \leq \epsilon$, we have $Q(\emptyset, a) =  0.75-1.5\lambda + \epsilon/2 > Q(\emptyset, b)$ by the assumption $\epsilon < 0.25$. Hence, NDL would, for any $\lambda$ prefer action $a$. However, for $\delta = 0.5$, CDP would prefer action $b$. Thus, for $\delta=0.5$, there is no $\lambda$ which make these equivalent. 
\end{proof}

\section{Historical smoothing and function approximation}
\label{app:smoothing}
The number of possible combinations (histories) grows exponentially with the number of actions, $k = |\cA|$. As a result, it is very probably that certain combinations of histories $h$ and actions $a$ are never observed in practice. We consider two solutions to this: historical smoothing and function approximation. Historical smoothing is used in the discrete case

by estimating the probability $p(Y(a)=y \mid H_{s-1}=h)$ using a weighted average of outcomes for observations $(h, a, y)$ and observations for subsequences $(h', a, y)$ where $h' \subseteq h$. Function approximation imputes $\hat{p}(Y(a)=y \mid H_{s-1}=h)$ using a regression estimator trained on all observations. We expand on these approaches in Appendix~\ref{app:smoothing}.

\subsection{Historical smoothing}
Consider estimating the function $p(Y(a)\mid H=h)$ in the discrete case. Under the stationarity assumption, Assumption~\ref{asmp:stationarity}, it is sufficient to represent the history in terms of indicators for tried treatments, $\{B_a \colon a\in \cA\}$ such that $B_a \in \{0,1\}$, and observed outcomes of these actions. Hence, $p(Y(a)\mid H=h)$ may be represented by a table of dimensions $|\cY| \times (\{0,1\} \times |\cY|)
^{|\cA|}$. Clearly, even under this representation, the number of possible histories grows exponentially with the number of actions. For this reason, for moderate to high numbers of actions, it will be unlikely to observe samples for each cell of this table. 

To obtain an estimate even in cases with high dimensionality, we use historical smoothing based on a  \textit{prior}. In the discrete case, we may view the  distribution of the outcomes $Y(a)$ for a treatment $a$ following history $h$ as a categorical distribution. We impose a Dirichlet prior on this distribution and use the posterior distribution in estimating the stopping statistic $\rho$ and in policy optimization. A Dirichlet prior for $p(Y(a)\mid H=h)$ is specified by pseudo-counts $\beta_1(a,h), ..., \beta_{|\cY|}(a,h)$. The posterior parameters are then
$\frac{n_y(a,h) + \beta_y(a,h)}{\sum_{y'} n_{y'}(a,h) + \beta_{y'}(a,h)}$, where $n_y(a,h)$ is equal to the number of samples where $Y(a)=y$ following history $h$. In this work, we consider two different priors $\beta$.

\paragraph{Historical prior (kernel smoothing)}
The historical prior assumes that the conditional  outcome distribution changes slowly with the number of past observations. The prior itself is a weighted average of the outcome probability at all possible previous histories,
\begin{equation}
    \label{eq:historical_prior}
    \beta_y(a,h) = \sum_{h' \subset h} w(h,h') \cdot \hat{p}(Y(a)\mid H=h'),
\end{equation}{}
where the weight of the probability given by a shorter history is determined by its similarity to $h$,
\begin{equation}
    \label{eq:historical_weight}
    w(h, h') = \frac{e^{-(|h|-|h'| - 1)^2}}{ |h| \cdot 2^{|h-h_i| - 1}}~.
\end{equation}

\paragraph{Uninformed prior}
The uninformed prior assigns a small uniform value to all $\beta$.

\subsection{Function approximation}
Observations for the $i$th subject are denoted $x^{(i)}, a^{(i)}_t, y^{(i)}_t, \ola^{(i)}_s$. To use function approximation, we fit a single function $f$, acting on a representation of history $\phi(h)$ to estimate $p(Y(a) \mid H=h)$ by solving the following problem,
\begin{equation}
\label{eq:functionappr}
\min_{f\in \cF}  \sum_{i=1}^n \sum_{s=1}^{t_i} L(f(h_s^{(i)}, a_s^{(i)}), y_s^{(i)})~,
\end{equation}
for an appropriately chosen function class $\cF$ and loss function $L$. In the discrete settings considered in the paper, we use the logistic (cross-entropy) loss which leaves the solution to \eqref{eq:functionappr} a probabilistic classifier, or estimate of $p(Y(a)\mid H=h)$ for all $a,h$.

\section{Additional experimental results}
Below follow additional details and results from the experiments. All experiments were implemented in Python and run on standard laptop computers. Each experiment on the synthetic DGP took less than a handful of hours to finish. For the antibiotics experiment, the overall time to produce the results for all values of $\delta$ was 2 days.

\subsection{Synthetic data generating process}
\label{app:synthetic}
We describe the datagenerating process (DGP) for the synthetic dataset used in Figure~\ref{fig:synthexp} and additional results described below. Let $o_a(h) = \mathds{1}[a \in h]$ and $o(h) = [o_1(h), ..., o_k(h)]^\top$. The moderator $Z \in \{0,1\}^d$ and covariates $X \in \{0,1\}^v$ are drawn according to 
\begin{enumerate}
\item $Z \sim \mbox{Bernoulli}(\alpha)$
\item $X \sim \mbox{Bernoulli}(\max(\min(\beta Z, 0.98), 0.02))$.
\end{enumerate}
given a set of parameters $\alpha \in [0,1]^d, \beta \in [0,1]^{v \times d}$ drawn element-wise uniformly at random.

The action \stop{} is drawn at any point following the first treatment with probability $p_{\stop{}} = 0.1$. To emulate a closer-to-realistic policy, if not stopped, the next action is drawn according to a categorical distribution with probabilities determined by the variable $X$ and the dissimilarity of the new action $A$ to previous actions in $H$. Outcomes are drawn according to a categorical distribution with parameters given by the pdf of a Cauchy random variable, itself with parameters depending on the variables $X$, $Z$ and $A$. For a full description of the data generating distribution, see Algorithm
~\ref{alg:synth}.

\begin{algorithm}[t!]
\SetAlgoLined
\textbf{Input: } Weight parameter $w_x$ (default value 1)\\
\textbf{Input: } Number of outcomes $n_y$ \\
\textbf{Input: } Uniform stopping probability $p_{\stop}$ \\
~\\
Generating parameters: \\
$u_1, u_2 \sim \mathcal{N}(0_{k \times (1+v+d)}, 1)$ \\
$u_2 \gets |u_2|$ \\
\For{$i\gets 2$ \KwTo $v+1$}{
    $u_1(\cdot,i) \gets u_1(\cdot,i) \cdot w_x$ \\
    $u_2(\cdot,i) \gets u_2(\cdot,i) \cdot w_x$ \\
}
$\eta \sim \mathcal{N}(0_{k \times (1+v+k)}, 1)$ \\
\For{$a\gets 1$ \KwTo $k$}{
    $u_1^-(a) \gets \sum_{i=1}^{1+v+d} \mathds{1}[u_1^-(a,i) < 0]u_1^-(a,i)$ \\
    $u_2^-(a) \gets \sum_{i=1}^{1+v+d} \mathds{1}[u_2^-(a,i) < 0]u_2^-(a,i)$ \\
    $u_1^+(a) \gets \sum_{i=1}^{1+v+d} \mathds{1}[u_1^-(a,i) > 0]u_1^-(a,i)$ \\
    $u_2^+(a) \gets \sum_{i=1}^{1+v+d} \mathds{1}[u_2^-(a,i) > 0]u_2^-(a,i)$ \\
}
~\\
Generating distribution of actions: \\
$p(A=\stop{}) = p_\stop$ \\
\For{$a, a' \in \{1, ..., k\}$}{
    $\Delta(a,a') \gets \|u_1(a) - u_1(a')\|_2^2 + \|u_2(a) - u_2(a')\|_2^2$
}
\For{$h \in \cH$}{   
    $v = [1 ; x ; o(h)]$ \\
    \For{$a \in \{1, ..., k\}$}{   
        $\tilde{p}(a) \gets e^{\eta(a,\cdot)^\top v}$
        \For{$a' \in h$}{   
            $\tilde{p}(a) \gets \tilde{p}(a)\cdot \Delta(a,a')$
        }
    }
    \For{$a \in \{1, ..., k\}$}{   
        $p(A=a \mid h, A\neq \stop{}) \gets \frac{\tilde{p}(a)}{\sum_{a\in \{1,...,k\}} \tilde{p}(a)}$
    }
}
~\\
Generating distribution of potential outcomes: \\
\For{$x \in \cX, z \in \cZ$}{   
    \For{$a\gets 1$ \KwTo $k$}{
        $v \gets [1 ; x ; z]$ \\
        $y_0(a) \gets u_1(a,\cdot)^\top v$ \\
        $y_0(a) \gets \frac{(n_y - 1)(y_0(a) - u_1^-(a))}{(u_1^+(a) - u_1^-(a))}$ \\
        $\gamma(a) \gets u_1(a,\cdot)^\top v$ \\
        $\gamma(a) \gets \frac{(\gamma(a) - u_2^-(a))}{(u_2^+(a) - u_2^-(a))}$ \\
        \For{$y\gets 1$ \KwTo $n_y$}{
            $\tilde{p}(a,y) \gets f_{\textnormal{cauchy}}(y; y_0(a), \gamma(a))$
        }
        \For{$y\gets 1$ \KwTo $n_y$}{
            $p(Y(a)=y \mid x, z) \gets \frac{\tilde{p}(a,y)}{\sum_{y=1}^{n_y} \tilde{p}(a,y)}$
        }
    }
}
\caption{Generating distribution of actions and potential outcomes\label{alg:synth}}
\end{algorithm}

\subsubsection{Additional results for the synthetic DGP}
We present additional results for \cdp{}, \cg{} and \ndp{} applied to the synthetic DGP described above. Unless otherwise specified, $\delta=0.4, \epsilon=0, \lambda=0.35$ and CDP and CG use the upper bound approximation of the stopping criterion described in Appendix~\ref{app:bound} with historical smoothing (\_H), as described in Appendix~\ref{app:smoothing}.

In Figure~\ref{fig:data_set_size_variance}, we illustrate the mean efficacy and search time (number of trials) as a function dataset size, varying logarithmically from $n=50$ to $n=75 000$ samples. We include the variance across $m$ random seeds for the experiment, $\hat{\sigma}^{2}=\frac{1}{m-1} \sum_{i=1}^{m}\left(x_{i}-\bar{x}\right)^{2}$. This Figure is a different view of Figure~\ref{fig:synthexp}, where we clearly see that the efficacy for most algorithms go up as data set size grows and search time decreases. For \ndp{}, as noted in Section~\ref{sec:experiments}, we see the opposite trend, however.

Figure~\ref{fig:deltasweep_tve} shows the trade-off between search time (number of trials) for different algorithms and 40 different values of $\delta \in [0,1]$ with $\lambda=\delta$ for $n=15 000$ samples, in the setting corresponding to Figure~\ref{fig:synthexp}. In Figure~\ref{fig:bounds_comparison}, we give the corresponding comparison for using lower or upper bounds in the estimation of the stopping criterion $\rho$, as described in Appendix
~\ref{app:bound}. Here, \_U refers to the upper bound, \_L to the lower bound and \_E is ``exact'' estimator, i.e. the empirical estimator of the exact expression for the stopping criterion, $\rho$. At first glance, the output of the different algorithms using different bounds appear very similar. However, as we see in Figure~\ref{fig:meaning_of_delta}, the  trade-off induced by a specific value of $\delta$ varies greatly depending on the estimation strategy. This is discussed also in Section~\ref{sec:experiments}, where we note that the policy learned by \ndp{} is very sensitive to the setting of $\lambda$. 

\begin{figure}
    \centering
    \begin{subfigure}{0.48\textwidth}
        \centering
        \includegraphics[width=\textwidth]{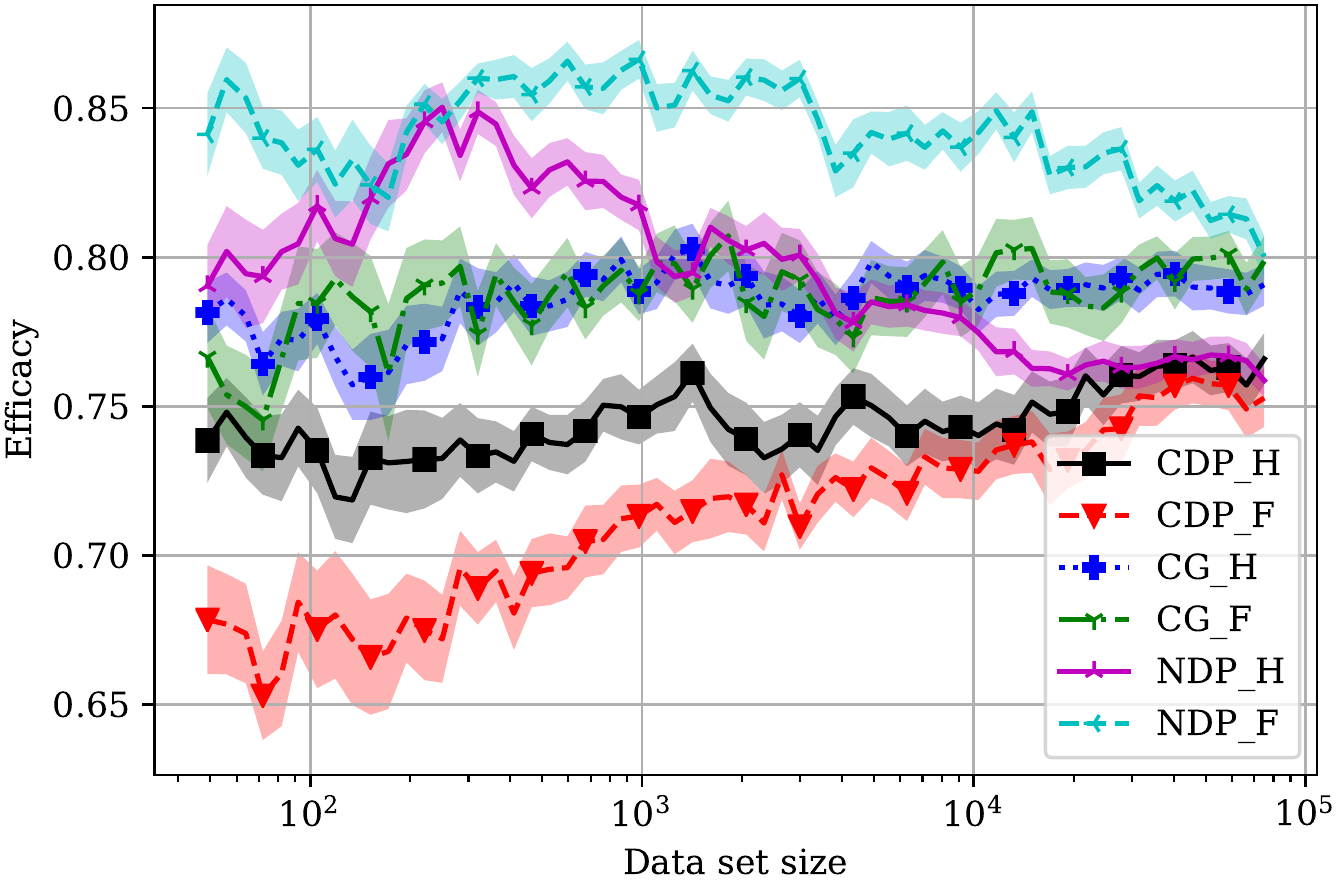}
        \caption{}
        \label{fig:efficacy_vs_data_set_size}
    \end{subfigure}
    \begin{subfigure}{0.48\textwidth}
        \centering
        \includegraphics[width=\textwidth]{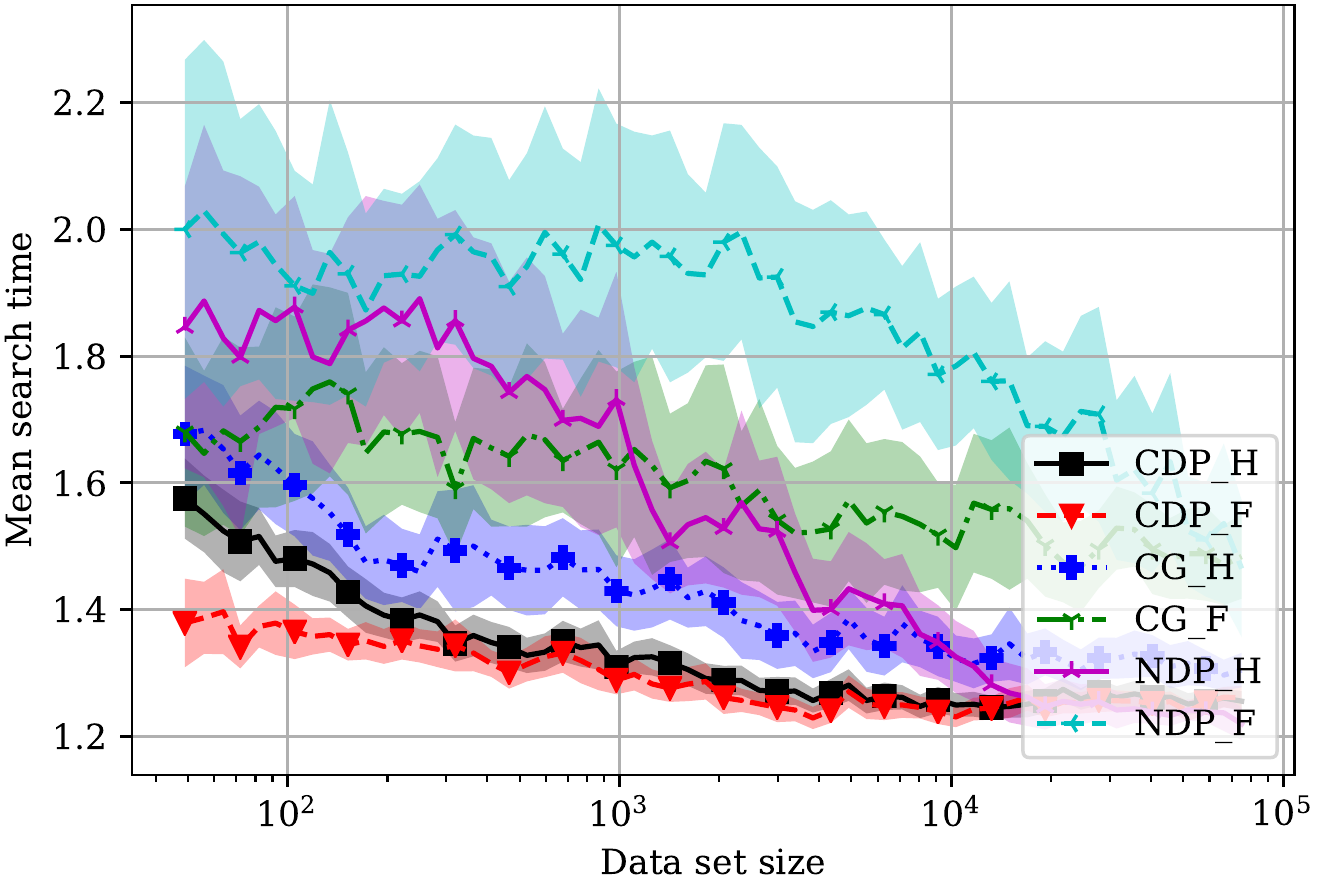}
        \caption{}
        \label{fig:mean_search_time_vs_data_set_size}
    \end{subfigure}
    \caption{Efficacy and time over different sized training sets for the synthetic DGP. Interval widths represent the variance across 50 realizations. }
    \label{fig:data_set_size_variance}
\end{figure}

\begin{figure}
    \centering
    \includegraphics[width=0.7\linewidth]{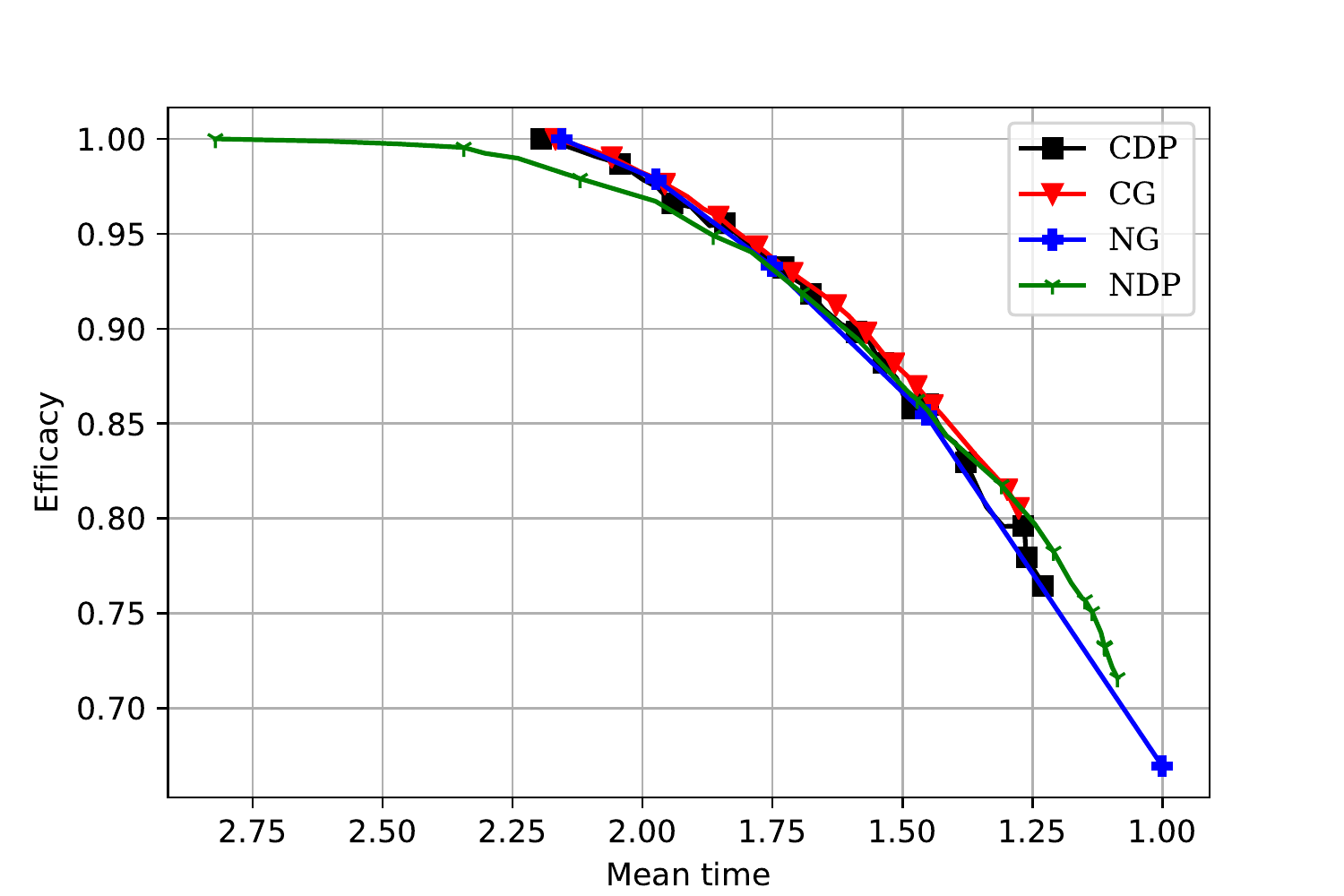}
    \caption{Efficacy and search time (number of trials) for different policy optimization methods operating the same model (historical smoothing, upper bound). }
    \label{fig:deltasweep_tve}
\end{figure}

\begin{figure}
    \centering
    \begin{subfigure}[t]{0.48\textwidth}
        \centering 
        \includegraphics[width=\linewidth]{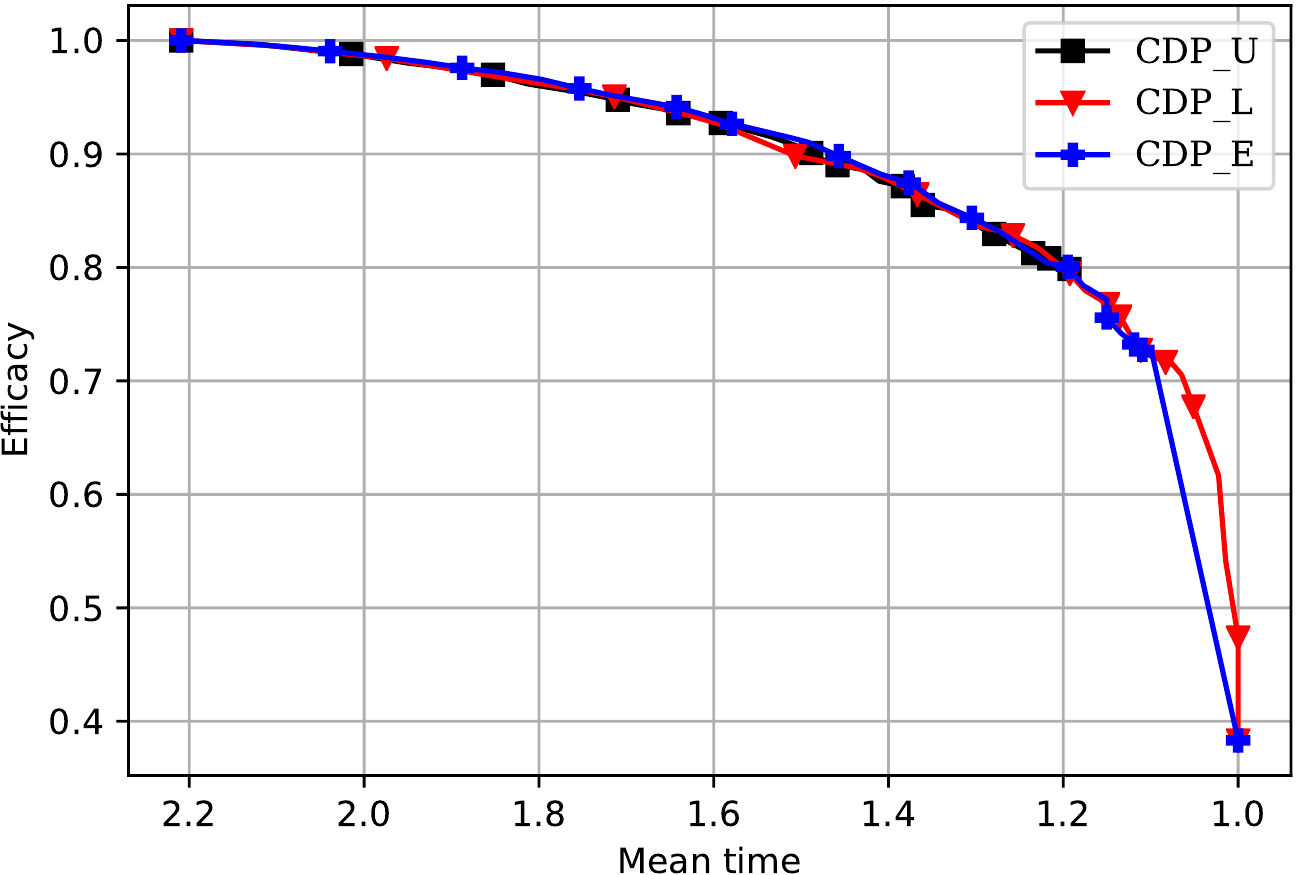}
        \caption{Constrained Dynamic Programming algorithm.}
        \label{fig:cdp_bound_tve}
    \end{subfigure}
    ~
    \begin{subfigure}[t]{0.48\textwidth}
        \centering
        \includegraphics[width=\linewidth]{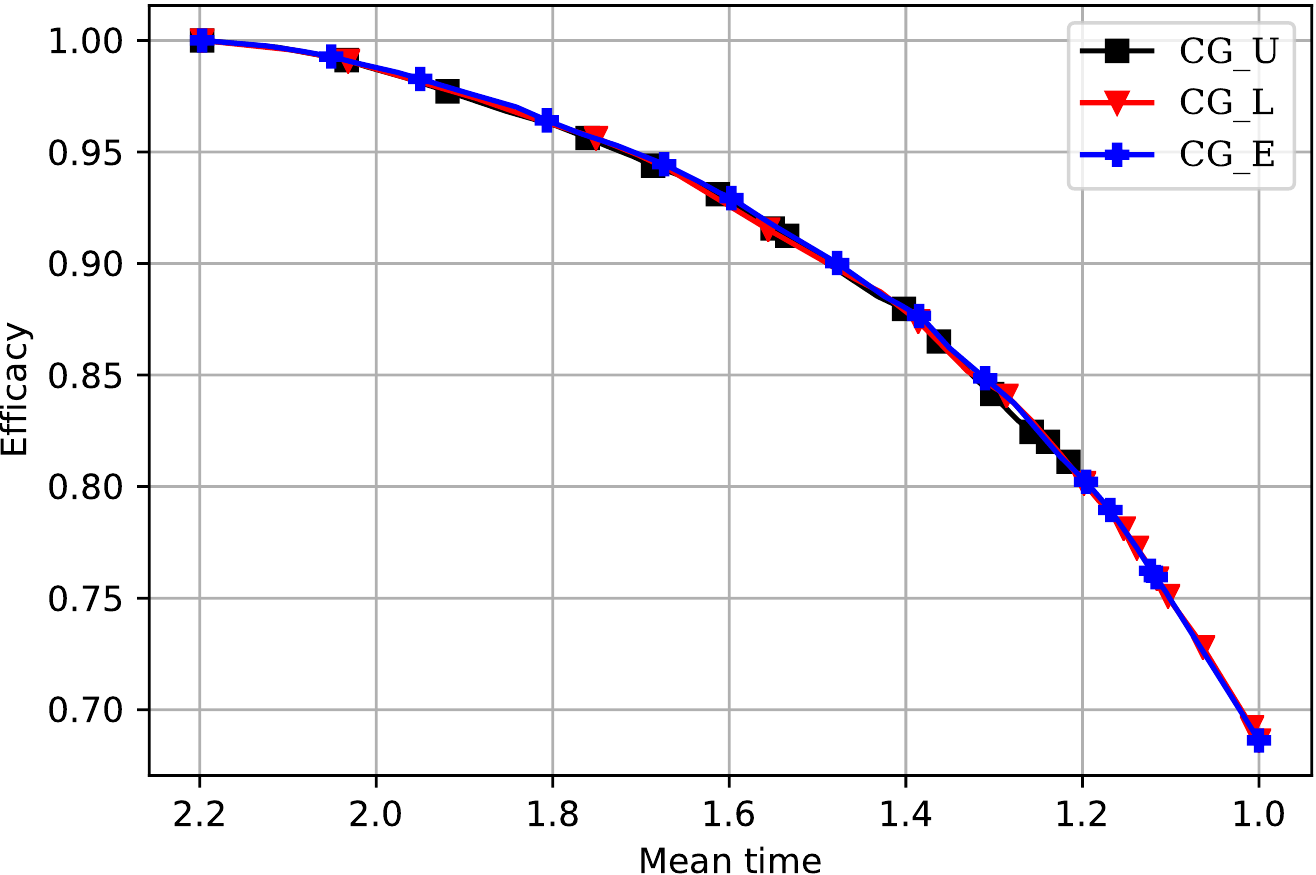}
        \caption{Constrained Greedy algorithm.}
        \label{fig:g_bound_tve}
    \end{subfigure}
    \caption{Results using estimates of the stopping criterion based on the upper (\_U) and lower bounds (\_L) described in Appendix~\ref{app:bound}, as well as the no-bound (exact) estimate (\_E) for the CDP and CG algorithms with $\delta$ varying linearly in $[0,1]$. }
    \label{fig:bounds_comparison}
\end{figure}

\begin{figure}
    \centering
    \begin{subfigure}[t]{0.48\textwidth}
        \centering
        \includegraphics[width=\textwidth]{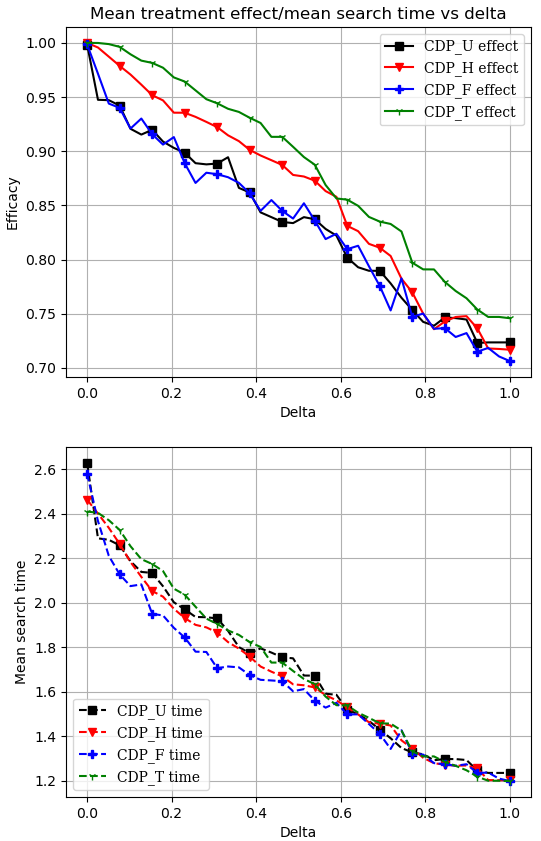}
        \caption{Efficacy and search time (number of trials) for varying approximations used in estimating the stopping criterion, with the upper bound, in the CDP algorithm. \_U stands for using a uniform prior to fill in missing valus. \_H is the historical kernel smoothing described in Appendix~\ref{app:smoothing}. \_F refers to function approximation and \_T the result for using the true model.}
        \label{fig:meaning_of_delta_approximators}
    \end{subfigure}
    ~
    \begin{subfigure}[t]{0.48\textwidth}
        \centering
        \includegraphics[width=\textwidth]{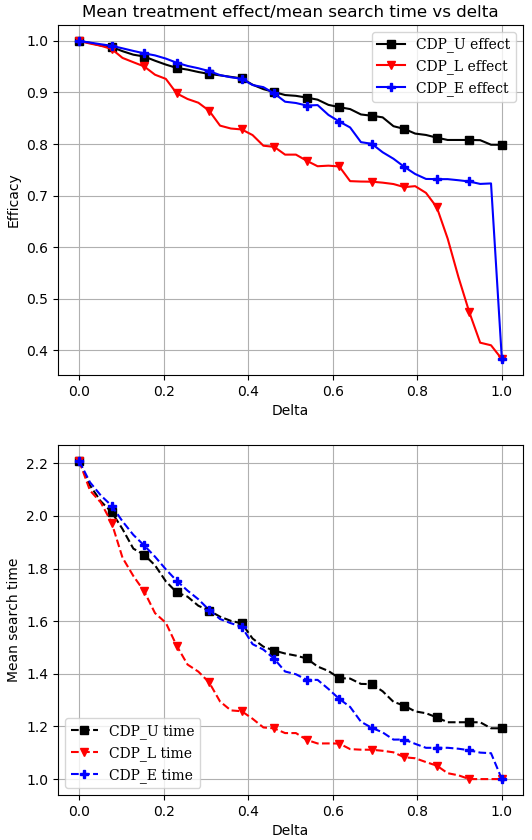}
        \caption{Efficacy and search time (number of trials) when using different bounds on the stopping criterion $\rho$ in the CDP algorithm. \_U stands for using the upper bound, \_L for the lower bound and \_E for the exact (no bound) estimate of $\rho{}$.}
        \label{fig:meaning_of_delta_bounds}
    \end{subfigure}
    \caption{Efficacy and mean search time (number of trials), varying $\delta$ in $[0,1]$.  }
    \label{fig:meaning_of_delta}
\end{figure}

\subsection{Antibiotic resistance dataset}
Below, we give additional information on the antibiotic resistance dataset compiled from MIMIC-III. 

To gather a cohort for which a consistent set of culture tests had all been performed for every patient, the set of organisms were restricted to a small subset. This selection was made based on overall prevalence in the data as well as the co-occurrence with common antibiotic culture tests. The selected organisms and antibiotics are listed below. 

\label{app:antibiotics}
\paragraph{Selected (bacterial) microorganisms:}
\begin{itemize}
\item Escherichia Coli (E. coli)
\item Pseudomonas aeruginosa
\item Klebsiella pneumoniae
\item Proteus mirabilis 
\end{itemize}

\paragraph{Selected antibiotics:}
\begin{itemize}
\item Ceftazidime
\item Piperacillin/Tazo
\item Cefepime
\item Tobramycin
\item Gentamicin
\item Meropenem
\end{itemize}
\textit{pending} was also an ``result'' in MIMIC-III, there were few of these instances and they were removed. 
Covariates X: Ages are divided into the four groups $[0,15]$, $(15,31]$, $(31,60]$, and $(60, \infty)$.
The two diseases are \textit{Infectious And Parasitic Diseases} and \textit{Diseases Of The Skin And Subcutaneous Tissue} as classified by ICD \citep{worldhealthorganization1978}.
The data was split in training and test 70/30 from 1362 patients and patients with multiple organisms were not split between the sets. Patients who had taken any antibiotic other than our chosen ones were not included in the data. Figure \ref{fig:antibiotics_delta_variance_app} uses the same data as Figure \ref{fig:func_vs_freq} but is split by $\delta$ and variance is shown.

\begin{tabular}{c|c}
    \# of treatments & \# of patients \\ \hline
    1 & 860 \\
    2 & 340 \\
    3 & 137 \\
    4 & 22 \\
    5 & 3 \\
\end{tabular}
\begin{figure}
    \begin{subfigure}{0.48\textwidth}
        \includegraphics[width=\textwidth]{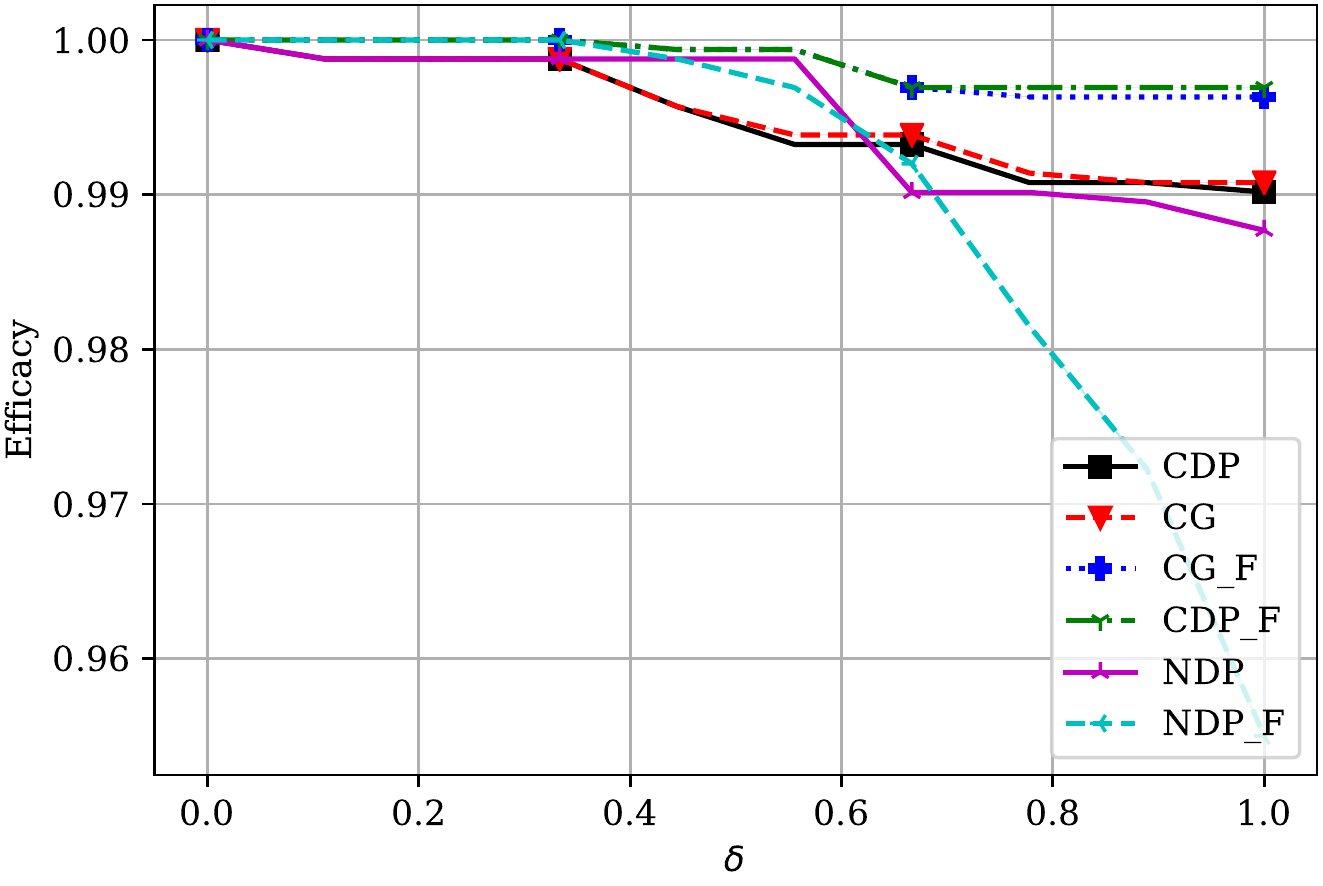}
        \caption{}
        \label{fig:antibiotics_efficacy_vs_delta_app}
    \end{subfigure}
    \begin{subfigure}{0.48\textwidth}
\includegraphics[width=\textwidth]{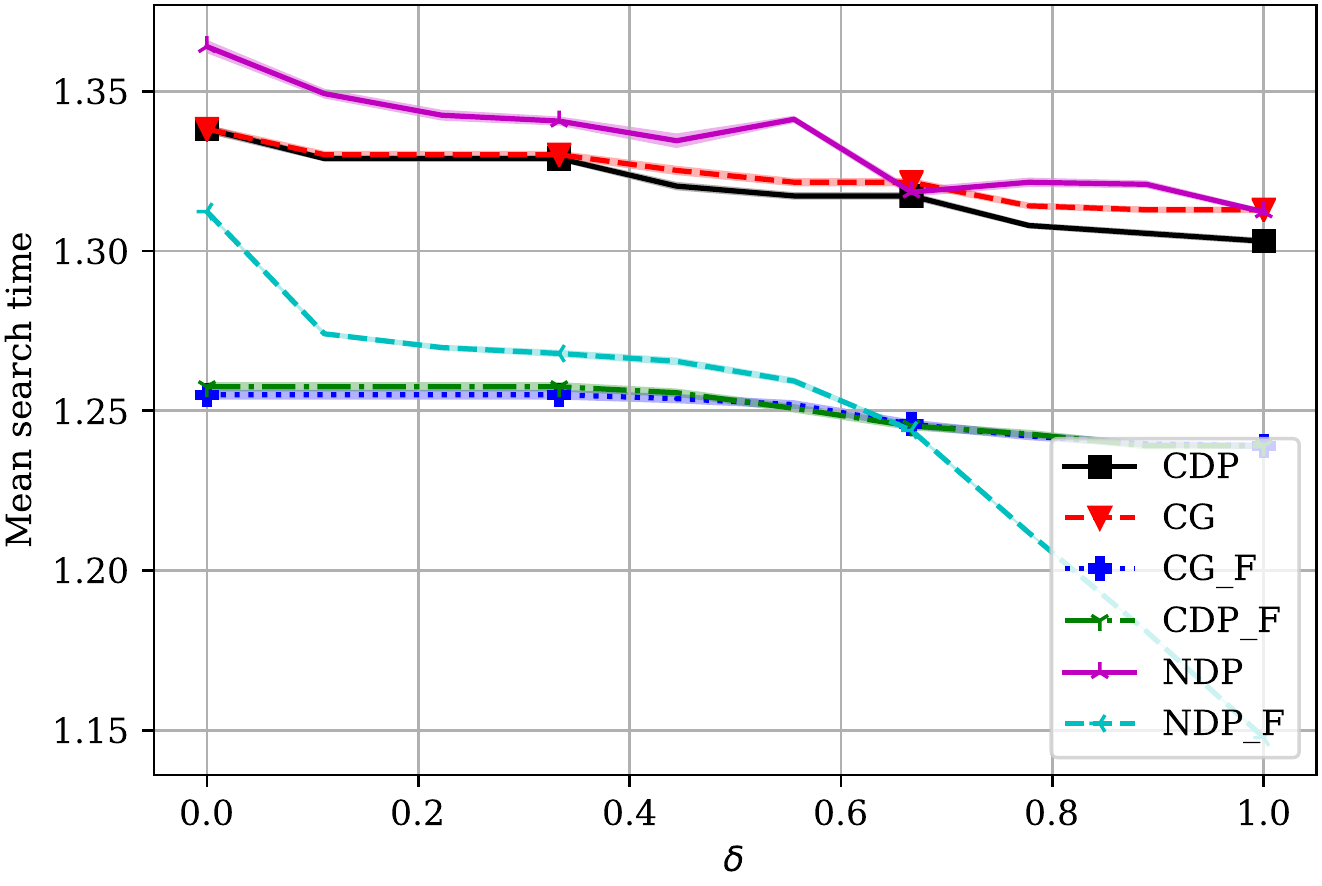}
        \caption{}
        \label{fig:antibiotics_mean_search_time_vs_delta_app}
    \end{subfigure}
    \caption{Efficacy and mean search time over different values of $\delta$ on the antibiotic resistance data set. The width of the plots represent the unbiased empirical sample variance across random splits.}
    \label{fig:antibiotics_delta_variance_app}
\end{figure}

\begin{figure}
    \begin{subfigure}{0.48\textwidth}
        \includegraphics[width=\textwidth]{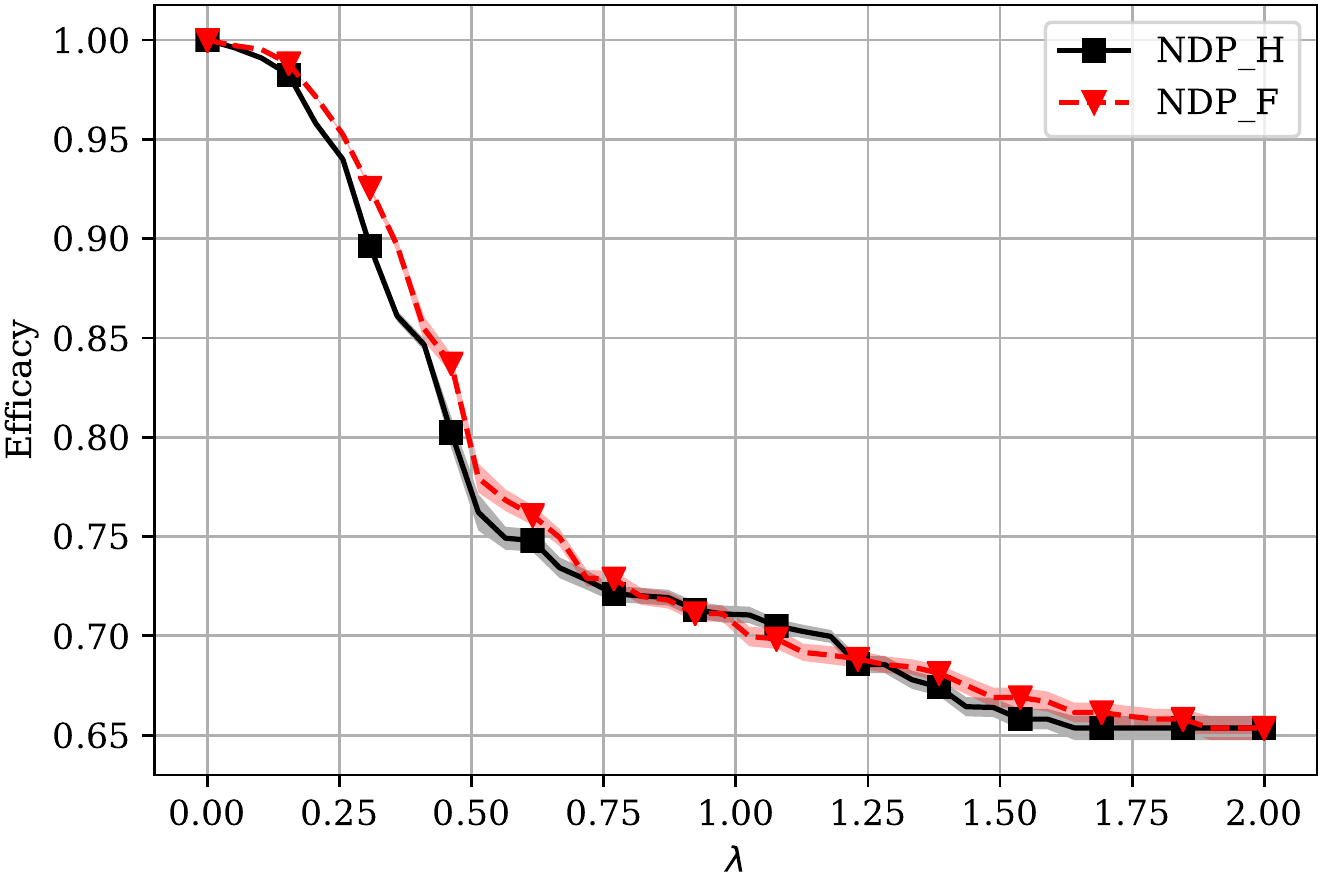}
        \caption{}
        \label{fig:antibiotics_efficacy_vs_lambda_app}
    \end{subfigure}
    \begin{subfigure}{0.48\textwidth}
\includegraphics[width=\textwidth]{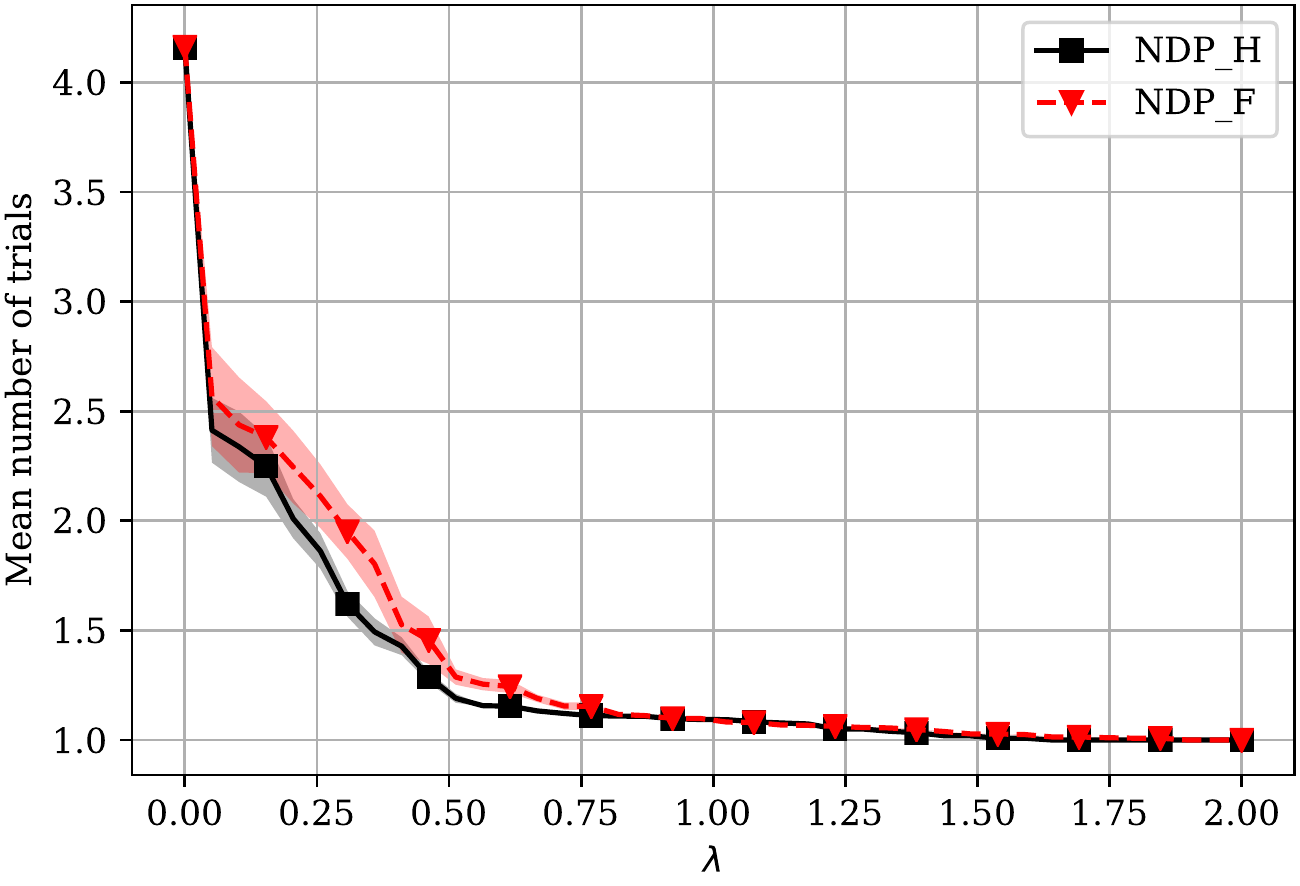}
        \caption{}
        \label{fig:antibiotics_mean_search_time_vs_lambda_app}
    \end{subfigure}
    \caption{Efficacy and mean number of trials over different values of $\lambda$ for the Naive Dynamic Programming algorithm. Variance is unbiased sample variance across random splits of the data. $\lambda$ is perturbed by 0.0001 in order to avoid division by zero for $\lambda = 0$. }
    \label{fig:antibiotics_vs_lambda_app}
\end{figure}


\end{document}